\documentclass[manuscript, screen]{jair}

\setcopyright{cc}
\copyrightyear{2026}
\acmDOI{10.1613/jair.1.21939}

\JAIRAE{J. Christopher Beck}
\JAIRTrack{} 
\acmVolume{86}
\acmArticle{9}
\acmMonth{6}
\acmYear{2026}

\RequirePackage[
  datamodel=acmdatamodel,
  style=acmauthoryear,
  backend=biber,
  giveninits=true,
  uniquename=init,
  ]{biblatex}

\newcommand{\deepcover}{{\sc d}eep{\sc c}over\xspace}
\usepackage{xspace}
\usepackage{amsmath}
\usepackage{mathtools}
\usepackage{amsthm}
\usepackage{nicefrac}
\usepackage{amsfonts}
\usepackage{pifont}
\usepackage{floatflt}
\usepackage{url}
\usepackage{dsfont}
\usepackage{comment}
\usepackage{tikz}
\usetikzlibrary{calc, positioning, shadows, backgrounds,arrows.meta,matrix}
\usepackage{multirow}
\usepackage{booktabs}
\usepackage{tabularx}
\usepackage{longtable}
\newcolumntype{L}{>{\raggedright\arraybackslash}X}

\newcommand{\lab}[1]{\textquotesingle{#1}\textquotesingle}

\usepackage[capitalize,noabbrev]{cleveref}

\usepackage{xcolor}
\usepackage{hyperref}
\usepackage{appendix}
\Crefname{appsec}{appendix}{appendices}

\usepackage{thmtools, thm-restate}

\theoremstyle{plain}
\newtheorem{theorem}{Theorem}[section]
\newtheorem{proposition}[theorem]{Proposition}
\newtheorem{lemma}[theorem]{Lemma}

\theoremstyle{definition}
\newtheorem{definition}[theorem]{Definition}

\theoremstyle{remark}

\theoremstyle{plain}
\newtheorem{observation}[theorem]{Observation}

\makeatletter
\DeclareRobustCommand*\cal{\@fontswitch\relax\mathcal}
\makeatother

\newcommand{\cet}{\textsc{\sc r}e{\sc x}\xspace}

\newcommand{\xai}{\textsc{XAI}\xspace}
\newcommand{\gradcam}{\textsc{Grad-CAM}\xspace}
\newcommand{\lime}{\textsc{lime}\xspace}
\newcommand{\plainshap}{{\sc shap}\xspace}
\newcommand{\shap}{\textsc{g}radient\plainshap}
\newcommand{\rise}{\textsc{rise}\xspace}
\newcommand{\extremal}{\textsc{extremal}\xspace}

\newcommand{\rex}{\textsc{\sc r}e{\sc x}\xspace}
\newcommand{\noise}{\textsc{n}oise\textsc{t}unnel\xspace}
\newcommand{\pytorch}{PyTorch\xspace}
\newcommand{\imagenet}{ImageNet\xspace}
\newcommand{\ks}{\textsc{KernelShap}\xspace}
\newcommand{\gs}{\textsc{GradientShap}\xspace}
\newcommand{\ie}{\emph{i.e.}\xspace}

\newcommand{\rap}{\textsc{rap}\xspace}
\newcommand{\lrp}{\textsc{lrp}\xspace}
\newcommand{\gbp}{\textsc{gbp}\xspace}
\newcommand{\ig}{\textsc{ig}\xspace}

\newcommand{\commentout}[1]{}

\newcommand{\Scal}{{\cal S}}

\newcommand{\U}{{\cal U}}

\newcommand{\cF}{{\cal F}}
\newcommand{\V}{{\cal V}}
\newcommand{\R}{{\cal R}}

\newcommand{\sat}{\models}

\newcommand{\fpa}{$\mbox{FP}^{{\rm A}[\log{n}]}$\xspace}
\newcommand{\fp}{$\mbox{FP}^{{\rm NP}[\log{n}]}$\xspace}

\newcommand{\lem}{\begin{lemma}}
\newcommand{\elem}{\end{lemma}}
\newcommand{\pro}{\begin{proposition}}
\newcommand{\epro}{\end{proposition}}

\newcommand{\dfn}{\begin{definition}}
\newcommand{\edfn}{\end{definition}}

\newtheorem{example}{Example}
\newcommand{\xam}{\begin{example}}
\newcommand{\exam}{\end{example}}

\newcommand{\eg}{e.g.\xspace}

\newcommand{\?}{\stackrel{?}{=}}
\usepackage{algorithm}
\usepackage{algorithmic}

\usepackage{longtable}
\usepackage{subcaption}
\usepackage{pgfplots}
\usepackage{pgfplotstable}
\usepackage{paralist}
\usepgfplotslibrary{fillbetween}
\pgfplotsset{compat=1.16}

\definecolor{tabutter}{rgb}{0.98824, 0.91373, 0.30980}          
\definecolor{ta2butter}{rgb}{0.92941, 0.83137, 0}               
\definecolor{ta3butter}{rgb}{0.76863, 0.62745, 0}               

\definecolor{taorange}{rgb}{0.98824, 0.68627, 0.24314}          
\definecolor{ta2orange}{rgb}{0.96078, 0.47451, 0}               
\definecolor{ta3orange}{rgb}{0.80784, 0.36078, 0}               

\definecolor{tachocolate}{rgb}{0.91373, 0.72549, 0.43137}       
\definecolor{ta2chocolate}{rgb}{0.75686, 0.49020, 0.066667}     
\definecolor{ta3chocolate}{rgb}{0.56078, 0.34902, 0.0078431}    

\definecolor{tachameleon}{rgb}{0.54118, 0.88627, 0.20392}       
\definecolor{ta2chameleon}{rgb}{0.45098, 0.82353, 0.086275}     
\definecolor{ta3chameleon}{rgb}{0.30588, 0.60392, 0.023529}     

\definecolor{taskyblue}{rgb}{0.44706, 0.56078, 0.81176}         
\definecolor{ta2skyblue}{rgb}{0.20392, 0.39608, 0.64314}        
\definecolor{ta3skyblue}{rgb}{0.12549, 0.29020, 0.52941}        

\definecolor{taplum}{rgb}{0.67843, 0.49804, 0.65882}            
\definecolor{ta2plum}{rgb}{0.45882, 0.31373, 0.48235}           
\definecolor{ta3plum}{rgb}{0.36078, 0.20784, 0.4}               

\definecolor{tascarletred}{rgb}{0.93725, 0.16078, 0.16078}      
\definecolor{ta2scarletred}{rgb}{0.8, 0, 0}                     
\definecolor{ta3scarletred}{rgb}{0.64314, 0, 0}                 

\definecolor{taaluminium}{rgb}{0.93333, 0.93333, 0.92549}       
\definecolor{ta2aluminium}{rgb}{0.82745, 0.84314, 0.81176}      
\definecolor{ta3aluminium}{rgb}{0.72941, 0.74118, 0.71373}      

\definecolor{tagray}{rgb}{0.53333, 0.54118, 0.52157}            
\definecolor{ta2gray}{rgb}{0.33333, 0.34118, 0.32549}           
\definecolor{ta3gray}{rgb}{0.18039, 0.20392, 0.21176}           

\begin{document}

\title[Causal Explanations for Image Classifiers]{Causal Explanations for Image Classifiers}

\author{Hana Chockler}
\authornote{Corresponding Author.}
\orcid{0000-0003-1219-0713}
\email{hana.chockler@kcl.ac.uk}
\affiliation{%
  \institution{King's College London}
  \city{London}
  \country{UK}
}
       
\author{David A. Kelly}
\orcid{0000-0002-5368-6769}
\email{david.a.kelly@kcl.ac.uk}
\affiliation{%
  \institution{King's College London}
  \city{London}
  \country{UK}
}
 
\author{Daniel Kroening}
\authornote{The work reported in this paper was done prior to joining Amazon.}
\orcid{0000-0002-6681-5283}
\email{dkr@amazon.com}
\affiliation{%
  \institution{Amazon.com, Inc.}
    \city{}
  \country{USA}
}
 
\author{Youcheng Sun}
\orcid{0000-0002-1893-6259}
\email{youcheng.sun@mbzuai.ac.ae}
\affiliation{%
  \institution{Mohamed bin Zayed University of Artificial Intelligence}
  \city{Abu Dhabi}
  \country{UAE}
  }
  \affiliation{
  \institution{University of Manchester}
  \city{Manchester}
  \country{UK}
}

\renewcommand{\shortauthors}{Chockler, Kelly, Kroening \& Sun}


\begin{abstract}
Existing algorithms for explaining the output of image classifiers 
use different definitions of explanations and a variety of techniques to find them. However, none of the existing
tools use a principled approach based on formal definitions of 
cause and explanation. 

In this paper we present a novel black-box approach to
computing explanations grounded in the theory of actual
causality. We prove relevant theoretical results and
present an algorithm for computing approximate explanations
based on these definitions. We prove termination of our
algorithm and discuss its complexity and the amount of
approximation compared to the precise definition.

We implemented the framework in a tool, \rex, and we present
experimental results and a comparison with state-of-the-art tools. 
We demonstrate that \rex is the most efficient black-box tool and produces the smallest explanations, in addition to outperforming
other black-box tools on standard quality measures. 
\end{abstract}

\received{13 February 2026}
\received[accepted]{25 March 2026}


\maketitle

\section{Introduction}
\label{sec:introduction}

\begin{figure}[htb]
    \centering
    \begin{subfigure}{0.2\linewidth}
        \centering
        \includegraphics[scale=0.35]{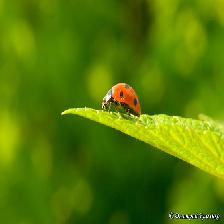}
        \caption{Ladybug}
        \label{fig:accept:ladybird}
    \end{subfigure}
    \hfill
    \begin{subfigure}{0.2\linewidth}
        \centering
        \includegraphics[scale=0.35]{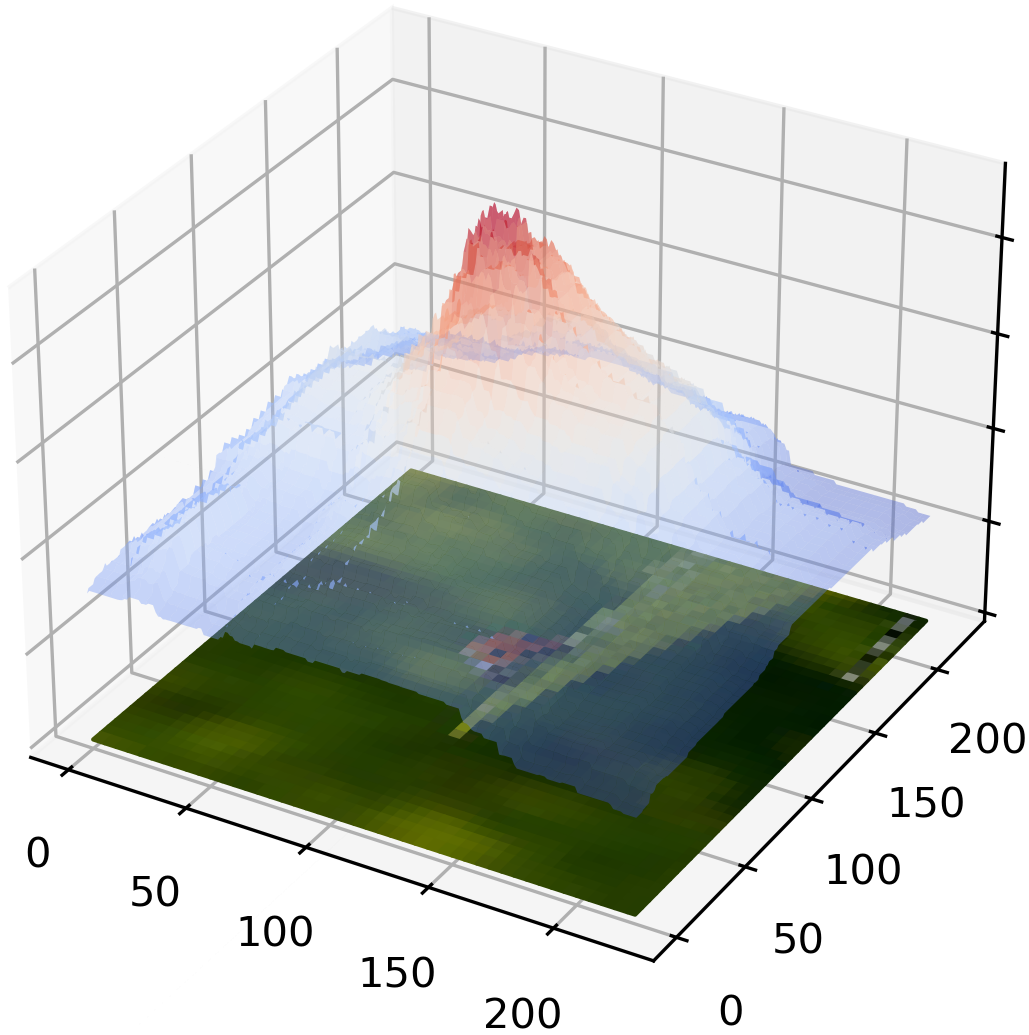}
        \caption{Responsibility}
        \label{fig:accept:surface}
    \end{subfigure}
    \hfill
    \begin{subfigure}{0.2\linewidth}
        \centering
        \includegraphics[scale=0.3]{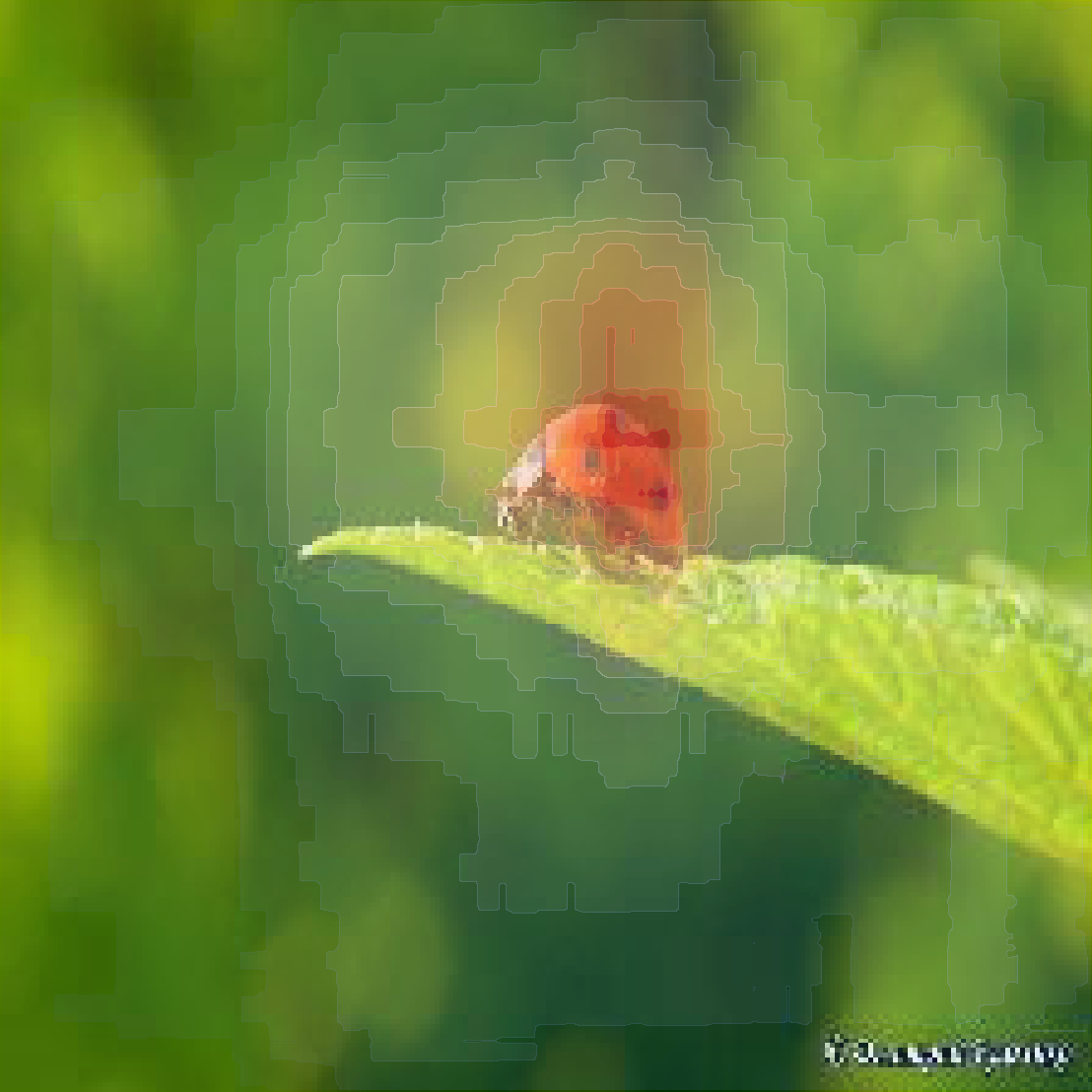}
        \caption{Heat map}
        \label{fig:accept:heat}
    \end{subfigure}
    \hfill
    \begin{subfigure}{0.2\linewidth}
        \centering
        \includegraphics[scale=0.35]{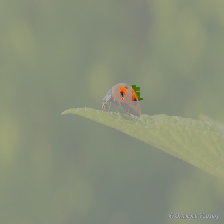}
        \caption{Explanation}
        \label{fig:accept:exp}
    \end{subfigure}
    \caption{A ladybug~(\subref{fig:accept:ladybird}), its responsibility map~(\subref{fig:accept:surface}), 
    the heat map~(\subref{fig:accept:heat}), which is a projection of
    the responsibility map on a plane overlaid on the original image, and a causal explanation~(\subref{fig:accept:exp}). 
    The minimal causal explanation computed by our tool \rex is less than $1\%$ of the image.}
    \label{fig:ladybird}
\end{figure}

Neural networks (NNs) are now a primary building block of many computer
vision systems. NNs are complex non-linear functions with algorithmically
generated coefficients. In contrast to traditionally
engineered image processing pipelines, it is difficult to retrace how the
pixel data are interpreted by the layers of the network.
Moreover, in many application areas, in particular healthcare,
the networks are proprietary, making the analysis of the internal layers impossible.
The ``black box" nature of neural networks creates demand for e\underline{X}plainable \underline{AI} (\xai) techniques that show
why a particular input yields the observed output without understanding the model's parameters and their influence on that output.

An explanation of an output of an automated procedure is essential in many areas, including verification,
planning, diagnosis and the like. An informative explanation may alter a user's confidence in the result. Explanations are also useful for determining whether there is a fault in the automated procedure:
if the explanation does not make sense, it may indicate that the procedure is faulty.
As argued by~\citet{bhusalface}, it is important to recognize that a ``good'' explanation should reveal what information a model is using to make its decision, even if that information is not intuitive to a human. 

There have been a number of definitions of explanations over the years in various domains of computer science~\citep{CH97,Gar88,Pea88}, philosophy~\citep{Hem65} and statistics~\citep{Sal89}. Black-box explanations for the results of image classifiers are typically based on or  are given in the form of a \emph{ranking} of the pixels, which is a numerical measure of importance: the higher the score, the more important the pixel is for the NN's classification outcome. Often these rankings are presented in a form of a \emph{heat map}, with higher scores corresponding to hotter areas.

This paper addresses the following research questions:
\begin{description}
    \item[RQ1] What is a formal, rigorous definition of explanation, suitable for analyzing image classifiers? 
    \item[RQ2] Can we compute explanations based on the definition above without opening the black box? What is the complexity of such computation?
    \item[RQ3] Is there an efficiently computable approximation for explanations? 
    \item[RQ4] What are suitable quality measures for explanations? 
    \item[RQ5] What is the quality of explanations computed by our algorithms compared to other \xai methods?
    \item[RQ6] Is there a trade-off between quality and compute cost of the explanations?
    \item[RQ7] Can black-box methods achieve the same quality of explanations as white-box methods?
\end{description}

Our algorithms are based on the formal definition of explanation in the theory of actual causality by~\citet{Hal19} and
its adaptation to image classifiers by~\citet{CH24}. Essentially, an explanation in the context of image classifiers is
a minimal subset of pixels with their original values that is sufficient for obtaining the same classification as the
original image, given that the rest of the image has its values replaced with data from a predefined dataset. We view a
black-box image classifier (a neural network) as a \emph{causal model} (\Cref{sec:cause}) of depth $2$ (that is, one
layer of input nodes, and one internal layer), with the classification itself computed in the single internal node, thus
capturing the opacity of the network. The algorithm (\Cref{sec:algorithm}) then calculates an approximate explanation,
where the approximation is in the size of the explanation (\Cref{fig:ladybird}). That is, an explanation computed by our
algorithm is sufficient for the classification, but it might not be minimal, though it is very close to minimal on real
images: in our experiments, explanations average between $3\%-5\%$ of the image, which is smaller than any of the other
black-box tools. Size is a useful metric here as it indicates the degree of extraneous information contained in the
explanation. We formally prove termination and complexity of our algorithm and discuss its approximation ratio.

We compare the experimental results of our implementation of the algorithm in a tool, \rex, with those of other \xai tools, both white-box and black-box, on a number of benchmark datasets: 
\imagenet, VOC2012, ECSSD, and a set of partially occluded images. Our results show that \rex is on par with the most efficient black-box tools in terms of execution time, while producing the smallest explanations. We also analyze the explanations using the
standard measures of insertion and deletion curves~\citep{petsiuk2018rise} and the intersection with the background or the occlusion. The results show that \rex
outperforms the other tools on insertion curves and on the overlap with irrelevant parts of the image. We discuss the logic behind deletion curves and
demonstrate that low deletion is a measure of the quality of the model for particular types of images, rather than for the quality of an explanation (\Cref{subsec:discussion}).
We also argue that explanations computed by \rex are so small that any further reduction in their size is not useful~\citep{kelly2025big}.

The rest of the paper is organized as follows. \Cref{sec:related} presents an overview of the related work on black-box explainable AI. \Cref{sec:cause} gives the
necessary background on actual causality. \Cref{sec:theory} provides the theoretical foundations of our approach to explainability. \Cref{sec:algorithm} is a detailed
overview of our algorithm, with the evaluation results discussed in \Cref{sec:evaluation}.  The tool \rex is
open source and available at \url{https://github.com/ReX-XAI/ReX}. All models and benchmark sets are publicly available and referenced in \Cref{sec:evaluation}.

\section{Related Work}
\label{sec:related}

The taxonomy of \xai methods and styles is large~\citep{schwalbe2024comprehensive}. Our method and its implementation in a tool, \rex, produces local, post hoc, black-box causal explanations of image classifications. By local, we mean that the explanation applies only to one image, not to all images with the same classification. This is in contrast to a global explanation, which explains an entire class.
Post hoc means that the method is applied to a model which is not inherently explainable and which has not been designed with explainability in mind. By black-box, we mean that there is no access at all to model architecture, gradients, or internals. An \xai method with free access to any layer of the model is \emph{white-box}. An \xai tool with access to discretionary information, such as the gradient, we term \emph{grey-box}. A proprietary system need not expose any information to the user other than the top-$1$ class. As our primary interest is in black-box, post hoc approaches, we only give a brief overview of other algorithms here. 

The existing approaches to post hoc explainability can largely be grouped into two categories: propagation and perturbation. 
Propagation-based explanation methods are usually more computationally efficient. They back-propagate a model's decision to the input layer to determine the weight of each input feature for the decision. \gradcam~\citep{CAM} only needs one backward pass and propagates the class-specific gradient into the final convolutional layer of a neural network to coarsely highlight important regions of an input image. 
Guided Back-Propagation (\gbp)~\citep{springenberg2015striving} computes the single and average partial derivatives of the output to attribute the prediction of a neural network. \lrp~\citep{LRP} uses layer-wise back-propagation applied to all layers of the model. The model's output is decomposed into different degrees of relevance by applying different rules for each layer.

In contrast to propagation-based explanation methods, perturbation-based explanation approaches explore the input space directly
in search of an explanation. The exploration/search often requires a large number of inference passes, which may incur significant computational cost, especially when compared to propagation methods. There are many ways to create these perturbations, most of which 
are based on random search or heuristics. The chief advantage of perturbation-based methods is that they are 
model independent and do not, in general, require any access to model internals.

Integrated Gradients (\ig)~\citep{sundararajan2017axiomatic} relies only on the model's classification and gradient. As a model does not need to expose the gradient, we consider \ig a \emph{grey-box} method.
It progressively introduces pixels from an image onto an underlying baseline value and computes an explanation based on the effect of these introductions on the output gradient. 

\plainshap (SHapley Additive exPlanations)~\citep{lundberg2017unified} is a family of different techniques, all of which seek to estimate Shapley values~\citep{winter2002shapley}. Shapley values measure the contribution, either positive or negative, of a feature towards an outcome. \shap is a gradient-based method that approximates Shapley values via gradient integration, in a similar fashion to \ig. There are numerous other methods for approximating Shapley values, among which \textsc{KernelShap}~\citep{lundberg2017unified} and \textsc{DeepLiftShap}~\citep{lundberg2017unified} are probably the most used. 
\textsc{DeepLiftShap} uses the \textsc{DeepLift}~\citep{deeplift} method whereas \textsc{KernelShap} uses the \lime framework.

Given a particular input, \lime~\citep{lime} samples its neighborhood and builds an \emph{inherently explainable} model to approximate the system's local behavior. An inherently explainable model for \lime is usually a variant on some linear regression model.
Owing to the high computational cost of this approach, \lime relies on a separate segmentation algorithm to break an image into 
\emph{superpixels}, that is, clusters of connected pixels that share similar characteristics, such as color, intensity, or texture. 
The quality of this initial segmentation is vital for good performance, as \lime does not refine it further~\citep{Knab2024BeyondPE}.
A \lime explanation is therefore necessarily tied to the provided segmentation. 
When this segmentation does not align with relevant features, \lime's performance suffers~\citep{blake2024explainable}. 
For instance, if the segmentation is too large or too crude, \lime's explanations contain irrelevant or unnecessary pixels. 
  
In \rise~\citep{petsiuk2018rise}, the importance of a pixel
is computed as the expectation over all local perturbations conditioned on the event that the pixel is observed. Rather than relying on an initial segmentation, \rise creates random occlusions of the input image. More recently, spectrum-based fault
localization (SBFL) has been applied to explaining image classifiers. The technique has been implemented in the tool \deepcover~\citep{sun2020explaining}.
While not an independent \xai method, \noise~\citep{noisetunnel} adds gaussian noise to the image before applying another \xai method, usually \ig. The noise serves to smooth the resultant saliency map. 

None of the methods mentioned above agree on what constitutes an explanation. \lime produces a locally explainable model, but the user is usually presented with a heat map of the local model's weights. \rise also produces a heat map, but its output is not directly comparable with \lime, due to the lack of a common definition. 
\shap output is also usually visualized as a heatmap, but one with both positive and negative values.

\rex, the method and tool presented in this paper, is a perturbation-based approach and addresses the limitations of existing black-box methods in two aspects. \rex does not rely on an initial segmentation as does \lime, nor does it simply rely on unguided random occlusions, as does \rise. The feature masking in \rex uses causal reasoning to guide the creation of occlusions. \rex output is also different from other methods: \rex shows the user which pixels are sufficient for a classification. \rex tests this sufficiency against the model itself, unlike \lime for example, which uses a locally trained model. Owing to its guided iterative refinement, \rex is computationally efficient, while still producing explanations that are state of the art.

\section{Background on Actual Causality}\label{sec:cause} 

In this section, we briefly review the definitions of causality and causal
models introduced by~\citet{HP01b} and
relevant definitions of causes and explanations in image
classification by~\citet{CH24}. The reader is referred
to~\citet{Hal19} for further reading.

We assume that the world is described in terms of 
variables and their values.  
Some variables may have a causal influence on others. This
influence is modeled by a set of {\em structural equations}.
It is conceptually useful to split the variables into two
sets: the {\em exogenous\/} variables, whose values are
determined by 
factors outside the model, and the {\em endogenous\/} variables, whose values are ultimately determined by
the exogenous variables.  
The structural equations describe how these values are 
determined.

Formally, a \emph{causal model} $M$
is a pair $(\Scal, \cF)$, where $\Scal$ is a \emph{signature}, which explicitly
lists the endogenous and exogenous variables  and characterizes
their possible values, and $\cF$ defines a set of \emph{(modifiable)
structural equations}, relating the values of the variables.  
A signature $\Scal$ is a tuple $(\U,\V,\R)$, where $\U$ is a set of
exogenous variables, $\V$ is a set 
of endogenous variables, and $\R$ associates with every variable $Y \in 
\U \cup \V$ a nonempty set $\R(Y)$ of possible values for 
$Y$ (i.e., the set of values over which $Y$ {\em ranges}).  
For simplicity, we assume here that $\V$ is finite, as is $\R(Y)$ for
every endogenous variable $Y \in \V$.
The set $\cF$ associates with each endogenous variable $X \in \V$ a
function denoted $F_X$
(i.e., $F_X = \cF(X)$) 
that expresses the dependency of $X$ on other variables.
If all variables in $M$ are Boolean, the model is called a \emph{binary model}.

The structural equations define what happens in the presence of external
interventions. 
Setting the value of some variable $X$ to $x$ in a causal
model $M = (\Scal,\cF)$ results in a new causal model, denoted
$M_{X\gets x}$, which is identical to $M$, except that the
equation for $X$ in $\cF$ is replaced by $X = x$.

Some salient features of causal models can be captured by causal graphs, where the nodes represent variables, and the
arrows depict the direction of causality.
A causal model  $M$ is \emph{recursive} if its causal graph is acyclic.
If $M$ is a recursive  causal model,
then given a \emph{context}, that is, a setting $\vec{u}$ for the
exogenous variables in $\U$, the values of all the other variables are
determined.
In this paper, following the convention, we restrict the discussion to recursive models.

We call a pair $(M,\vec{u})$ consisting of a causal model $M$ and a
context $\vec{u}$ a \emph{(causal) setting}.
A causal formula $\psi$ is true or false in a setting.
We write $(M,\vec{u}) \sat \psi$  if
the causal formula $\psi$ is true in
the setting $(M,\vec{u})$.
The $\sat$ relation is defined inductively.
$(M,\vec{u}) \sat X = x$ if
the variable $X$ has value $x$
in the unique solution
to the equations in
$M$ in context $\vec{u}$.
Finally, 
$(M,\vec{u}) \sat [\vec{Y} \gets \vec{y}]\varphi$ if 
$(M_{\vec{Y} = \vec{y}},\vec{u}) \sat \varphi$,
where $M_{\vec{Y}\gets \vec{y}}$ is the causal model that is identical
to $M$, except that the 
variables in $\vec{Y}$ are set to $Y = y$
for each $Y \in \vec{Y}$ and its corresponding 
value $y \in \vec{y}$. The setting of $\vec{Y}$ to $\vec{y}$ effectively cuts the incoming causal transitions into $\vec{Y}$, replacing the variables
in $\vec{Y}$ with the values $\vec{y}$. These values are then propagated to the variables that causally depend on $\vec{Y}$. This change of the values of some
of the variables is called an \emph{intervention}.

A standard use of causal models is to define \emph{actual causation}: that is, 
what it means for some particular event that occurred to cause 
another particular event. 
There have been a number of definitions of actual causation given
for acyclic models
(\eg~\citep{beckers21c,GW07,Hall07,HP01b,Hal19,hitchcock:99,Hitchcock07,Weslake11,Woodward03}).
In this paper, we are focusing on \emph{explainability}, which is traditionally defined using some
form of actual causation. As we argue in the next section, in our setting, the causal models have a
simple structure, and hence all the relevant definitions are simplified. 

\section{Theoretical Foundations of Our Approach}\label{sec:theory}

We approach the first two research questions theoretically.

\begin{description}
    \item[RQ1] What is a formal rigorous definition of explanation, suitable for image classification? 
\end{description}

Given an image classifier (e.g., a neural network) $\mathcal{N}$ and an input image $x$, we define a
binary causal model $M_{\mathcal{N},x}$ as follows. The set $\V = \vec{V} \cup \{O\}$ of endogenous variables consists
of a set $\vec{V}$ corresponding to the set of pixels $P(x)$ of $x$ and the single output variable $O$. 
Essentially, $\vec{V}$ is a \emph{mask}, indicating which pixels of $x$ are visible and which are occluded, and
the output variable $O$ indicates whether the classification of a partially masked image stays the same as of the original image.
To keep in line with the formal definitions of structural causal models, the values of $\vec{V}$ are determined by 
the set of exogenous variables $\U$.

Assigning $1$ to a variable $v_i \in \vec{V}$ means that the pixel $p_i$, corresponding
to $v_i$, has its original value (taken from $x$). Assigning $0$ to this variable means that $p_i$ is masked --
replaced with some predefined masking value. 
The masking operation of $\V$ on $P(x)$ is denoted by $\V \odot P(x)$ and
is the Hadamard product of these sets viewed as matrices of the same size (corresponding to the input size and shape of $\mathcal{N}$).

The context $\vec{u}$ that assigns all variables in $\vec{V}$ the value $1$ (i.e., 
none of the pixels are masked) corresponds to the fully unmasked image $x$. 
The value of $O$ is $1$ iff the output of $\mathcal{N}$ on $\V \odot P(x)$,
the partially masked image defined by applying $\vec{V}$ to $x$, is $\mathcal{N}(x)$ and is $0$ otherwise. 
Clearly, $O=1$ if the image is fully unmasked, that is, $(M_{\mathcal{N},x},\vec{u}) \models (O=1)$.

\begin{figure}[t]
    \centering
    \begin{tikzpicture}[outer sep=auto]
    \node (V) at (-1, 0) {$\vec{V}$};
    \node (v1) at (0, 0) [draw, circle, minimum size=0.9cm,fill=red!10] {$v_1$};
    \node (v2) at (1.5, 0) [draw, circle, minimum size=0.9cm,fill=red!10] {$v_2$};
    \node (v3) at (3, 0) [draw, circle, minimum size=0.9cm,fill=red!10] {$v_3$};
    \node (v4) at (4.5, 0) [draw, circle, minimum size=0.9cm,fill=red!10] {$v_4$};
    \node (v5) at (6, 0) [draw, circle, minimum size=0.9cm,fill=red!10] {$v_5$};
    \node (dots) at (7, 0) {$\mathbf{\cdots}$};
    \node (v6) at (8, 0) [draw, circle, minimum size=0.3cm,fill=red!10] {$v_{n-2}$};
    \node (v7) at (9.5, 0) [draw, circle, minimum size=0.3cm,fill=red!10] {$v_{n-1}$};
    \node (vn) at (11, 0) [draw, circle, minimum size=0.9cm,fill=red!10] {$v_n$};

    \node (f) at (5.5, -2.2) [draw, rectangle, rounded corners, fill=blue!10, inner sep=0.25cm] 
    {$\mathcal{N}(\vec{v} \odot P(x)) \? \mathcal{N}(x)$};

    \draw [-Triangle] (v1) -- (f);
    \draw [-Triangle] (v2) -- (f);
    \draw [-Triangle] (v3) -- (f);
    \draw [-Triangle] (v4) -- (f);
    \draw [-Triangle] (v5) -- (f);
    \draw [-Triangle] (v6) -- (f);
    \draw [-Triangle] (v7) -- (f);
    \draw [-Triangle] (vn) -- (f);
    
    \node (o) at (5.5, -3.8) [draw, circle, minimum size=0.9cm, fill=yellow!10] {$O$};
    \node (oin) at (7, -3.8) {$O \in \{0,1\}$};

    \draw [-Triangle] (f) -- (o);
    \end{tikzpicture}
    \caption{A depth-$2$ binary causal model $M_{\mathcal{N},x}$ for an image $x$ and a classifier $\mathcal{N}$. 
    $\vec{v}$ is the vector of values of $\vec{V}$.
    The output $O \in \{0, 1\}$ indicates
    whether the classification of the Hadamard product of $x$ and $\vec{V}$ is the same as the original classification.}
    \label{fig:MNx}
\end{figure}

We depict the reasoning process of $M_{\mathcal{N},x}$ in \Cref{fig:MNx}, omitting the set of exogenous variables.
The causal model has depth $2$. In what follows, we omit the subscript ${\mathcal{N},x}$ from the causal model notation
if it is clear from the context.
This construction assumes causal independence between the variables in $\vec{V}$. This is common 
to many approaches to causal and counterfactual explainable AI~\citep{USL19,SHG20,PSSBF20,MMTS21,Beckers22,CH24}, and is the \emph{de facto} approach in all black-box explainable AI tools. We argue, however, that this is an accurate depiction in the case
of images, and not just a convenient approximation. Image classification models perform classifications over data, not over what the image represents. This data encodes correlations which are due to the underlying data production method being causal, \ie the real world. Obscuring or permuting pixels does not lead to a cascading change of other pixels values, as one would expect if pixels were causally connected, instead these permutations disrupt correlations in the data only.

Consider the ``Photobombing'' dataset, described in \Cref{sec-exp-setup}. This dataset is comprised of partially obscured images, obtained by overlaying random color patches over Imagenet images. An input from
this dataset is a perfectly valid image. Indeed, obscuring a part of the input image either by introducing an artificial object  or by positioning a real object in front of the primary subject of the classification 
does not lead to any change in unobscured pixels. Thus, pixel independence holds on the general data representation of images.


We refer the reader
to \citet{CH24} for a more in-depth discussion of causal independence between pixel values in
image classification.  
We note that the causal independence assumption may not hold for concept bottleneck models and other models which claim to learn causes and not just correlations~\citep{kohconcept}. For these types of models, assuming independence
may result in approximation which might lead to inaccurate results.

As we discuss below, there is a strong \emph{correlation} between pixels in a given image, which is not surprising given that images usually capture objects in the real
world. This does not contradict our claim of the lack of causal connection between the pixels in the image, due to the argument above.



Given a neural network $\mathcal{N}$ and an input image $x$, let $\vec{u}_1$ be the context
that assigns $1$ to all variables in $\vec{V}$, and let $\vec{u}_0$ be the context
that assigns $0$ to all these variables. 
We introduce the following definition of explanation.

\begin{restatable}[Single-Context Explanation]{definition}{expdef}\label{defn:simple-exp}
A subset $\vec{V}_{exp}$ of $\vec{V}$ is a \emph{single-context explanation} of a classification $\mathcal{N}(x)$ of 
an input image $x$ by a classifier 
$\mathcal{N}$ if the following conditions hold:
\begin{description}
\item[{\rm EXIM1.}] $(M,\vec{u}_0) \models [\vec{V}_{exp} = 1](O=1)$.
\item[{\rm EXIM2.}] $\vec{V}_{exp}$ is minimal; there is no strict subset $\vec{V}'_{exp}$ of $\vec{V}_{exp}$ that satisfies EXIM1.
\end{description}
When there is no confusion, we call \emph{single-context explanation} simply an \emph{explanation}.
As there is a one-to-one correspondence between the variables in $\vec{V}$ and the pixels of $x$, 
we also call the subset of pixels $P_{exp}$ of $x$ that corresponds to $\vec{V}_{exp}$ a \emph{single-context explanation} of $\mathcal{N}(x)$.
\end{restatable}
In other words, an explanation is a minimal subset of pixels of a given input image $x$ that is sufficient for the 
model $\mathcal{N}$ to classify the image, with all other pixels masked.
Note that we do not assume that the classification of a fully masked input by $\mathcal{N}$ is different from $\mathcal{N}(x)$; if they are equal, the (single)
explanation is an empty set.

\Cref{defn:simple-exp} matches other definitions of explanations in the causality landscape. So as not to interrupt the flow for the reader, we define and
prove the relevant equivalences in \Cref{app:exp}. The reader can safely ignore this appendix, as it is not directly related to the methods and algorithms
presented in this paper.

We now return to the task at hand, which is \emph{computing} explanations.
\begin{description}
    \item[RQ2] Can we compute explanations based on the definition above without opening the black box? What is the complexity of such computation?
\end{description}

Unfortunately, the precise computation of explanations is intractable---see \Cref{app:exp} for the proof. 
It is of little surprise, as all the existing definitions of explanations
are intractable, except in some very simple special cases. 
There is, therefore, a justification for an approximation algorithm for explanations. \Cref{sec:algorithm}, described in the next section,  
computes approximate explanations. It uses an approximate ranking of pixels according
to their importance for the classification of the input image as a heuristic to guide explanation discovery. 
We formalize the notion of importance as follows.

\begin{restatable}[Sufficient responsibility]{definition}{respdef}\label{def:simple-resp}
The \emph{degree of sufficient responsibility} of an explanation $\vec{V}_{exp}$ for the classification of $x$
by $\mathcal{N}$ is defined as $1/|\vec{V}_{exp}|$. We also extend this value to all pixels in $\vec{V}_{exp}$. That is, $v_i$ (and its matching pixel $p_i$) 
have the degree of sufficient responsibility $1/|\vec{V}_{exp}|$ for the classification of $x$
by $\mathcal{N}$, where $\vec{V}_{exp}$ is the smallest explanation
for the classification of $x$ that contains $v_i$. If there is no explanation that contains $v_i$, then
its degree of sufficient responsibility is defined as $0$.
\end{restatable}

Similarly to the definition of explanation, \Cref{def:simple-resp} has a match in the existing causality landscape, as we
show in \Cref{app:exp}, where we also prove its intractability.
The following observation describes a brute-force approach to computing sufficient responsibility, which
is clearly exponential in the number of pixels of the image.
\begin{observation}
Given an image $x$ and its classification $o$, we can calculate the degree of sufficient responsibility of each pixel $p_i$ of $x$
by directly applying~\Cref{defn:simple-exp}, that is, by checking the conditions EXIM1 and EXIM2
for all subsets of pixels of $x$. 
\end{observation}

In the next section, we describe an efficient approximation algorithm for computing explanations that is based on an efficiently computable approximate degree
of responsibility.

\section{The \rex Algorithm}
\label{sec:algorithm}

\begin{figure}
\tikzset{
   arrowshadow/.style = { thick, color=black, ->, double copy shadow={thick,shadow xshift=2ex, shadow yshift=2ex}},
}

\begin{tikzpicture}[node distance=5.5cm]


\definecolor{tabutter}{rgb}{0.98824, 0.91373, 0.30980}		
\definecolor{ta2butter}{rgb}{0.92941, 0.83137, 0}		
\definecolor{ta3butter}{rgb}{0.76863, 0.62745, 0}		

\definecolor{taorange}{rgb}{0.98824, 0.68627, 0.24314}		
\definecolor{ta2orange}{rgb}{0.96078, 0.47451, 0}		
\definecolor{ta3orange}{rgb}{0.80784, 0.36078, 0}		

\definecolor{tachocolate}{rgb}{0.91373, 0.72549, 0.43137}	
\definecolor{ta2chocolate}{rgb}{0.75686, 0.49020, 0.066667}	
\definecolor{ta3chocolate}{rgb}{0.56078, 0.34902, 0.0078431}	

\definecolor{tachameleon}{rgb}{0.54118, 0.88627, 0.20392}	
\definecolor{ta2chameleon}{rgb}{0.45098, 0.82353, 0.086275}	
\definecolor{ta3chameleon}{rgb}{0.30588, 0.60392, 0.023529}	

\definecolor{taskyblue}{rgb}{0.44706, 0.56078, 0.81176}		
\definecolor{ta2skyblue}{rgb}{0.20392, 0.39608, 0.64314}	
\definecolor{ta3skyblue}{rgb}{0.12549, 0.29020, 0.52941}	

\definecolor{taplum}{rgb}{0.67843, 0.49804, 0.65882}		
\definecolor{ta2plum}{rgb}{0.45882, 0.31373, 0.48235}		
\definecolor{ta3plum}{rgb}{0.36078, 0.20784, 0.4}		

\definecolor{tascarletred}{rgb}{0.93725, 0.16078, 0.16078}	
\definecolor{ta2scarletred}{rgb}{0.8, 0, 0}			
\definecolor{ta3scarletred}{rgb}{0.64314, 0, 0}			

\definecolor{taaluminium}{rgb}{0.93333, 0.93333, 0.92549}	
\definecolor{ta2aluminium}{rgb}{0.82745, 0.84314, 0.81176}	
\definecolor{ta3aluminium}{rgb}{0.72941, 0.74118, 0.71373}	

\definecolor{tagray}{rgb}{0.53333, 0.54118, 0.52157}		
\definecolor{ta2gray}{rgb}{0.33333, 0.34118, 0.32549}		
\definecolor{ta3gray}{rgb}{0.18039, 0.20392, 0.21176}		

\definecolor{pantonegreen}{HTML}{669900}
\definecolor{pantoneblue}{HTML}{0066CC}
\definecolor{pantonered}{HTML}{CC3333}
\definecolor{pantoneorange}{HTML}{FFCC33}

  \node (img) {\includegraphics[width=2cm]{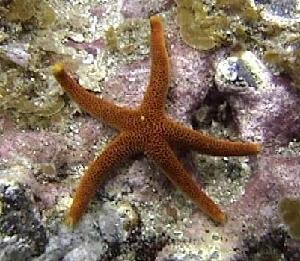}};

  \node[rectangle, rounded corners, top color=pantonegreen!30, bottom color=pantonegreen!30, minimum height=1.75cm, right =0.3cm of img, align=center, draw=ta2chameleon] (ra) {{\large\ding{172}}\\Causal\\Ranking\\ Algorithm};

  \node[right=0.25cm of ra] (map) {\includegraphics[width=2.5cm]{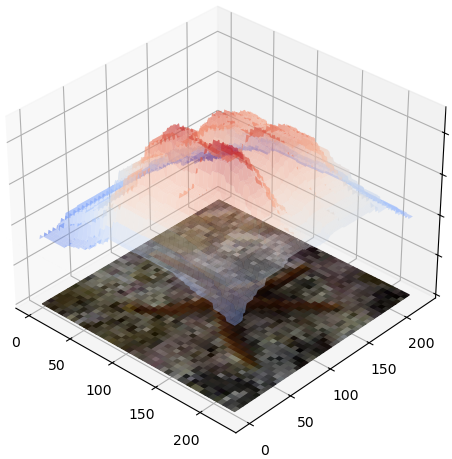}};
  \draw[->, thick] (img) -- (ra);

   \begin{scope}[on background layer]
            \draw[rounded corners, fill=red!10] (1.2, 1.5) rectangle ++(5.25, -3);
        \end{scope}

  \node[rectangle, rounded corners, align=center, draw=black!50!pantoneorange, fill=pantoneorange!30, minimum height=2cm, right=1cm of map] (extract) {\large\ding{173}\\Explanation\\extraction};

  \node[right=0.25cm of extract] (exp) {\includegraphics[width=2cm]{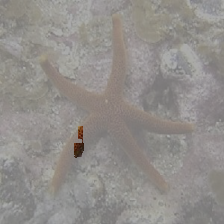}};

  \begin{scope}[on background layer]
            \draw[rounded corners, fill=green!10] (7, 1.5) rectangle ++(4.8, -3);
        \end{scope}
  \draw[->, thick] (map) -- (extract);

  \node[inner sep=0pt, right=0.5cm of exp] (final) {\includegraphics[scale=0.08]{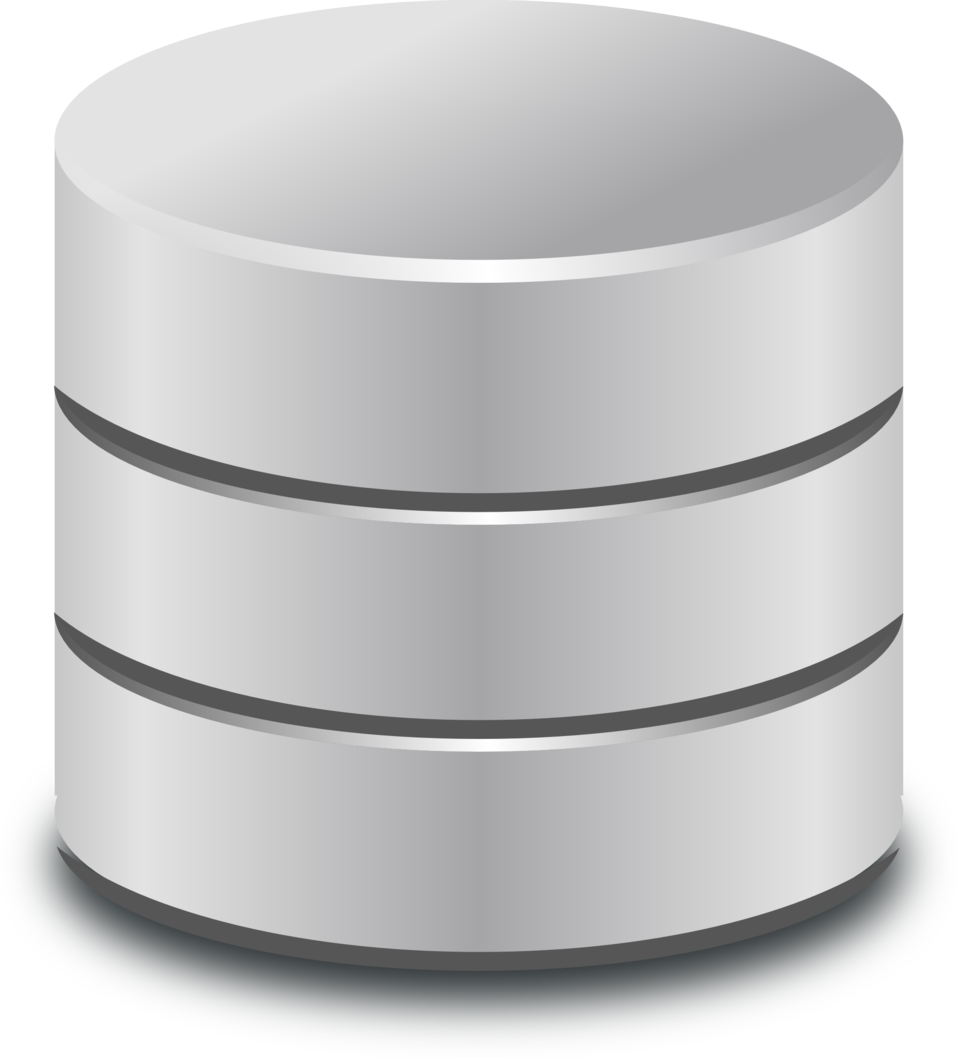}};
  \draw[->, thick] (exp) -- (final);

\end{tikzpicture}
\caption{High level overview of \Cref{algo:compositonal_explanation}. The causal ranking algorithm produces an approximate 
responsibility map (\ding{172}). The pixels in the image are then ordered by
their approximate responsibility, and the explanation extraction algorithm uses this ranking to produce an approximately minimal 
sufficient explanation (\ding{173}), which captures the information required for the DNN to give the classification
(\lab{starfish} in this example).}%
\label{fig:algo}
\end{figure}

In this section we introduce our algorithm for computing approximate explanations. The high-level view of the algorithm is
presented in \Cref{fig:algo}. The algorithm consists of two independent parts. The first part ranks
pixels according to their approximate degree of sufficient responsibility (\Cref{def:simple-resp}); the second part is a 
\emph{greedy algorithm} that extracts (approximate) explanations from this ranking.

The algorithm for computing an approximate degree of responsibility is based on the notion of a \emph{superpixel} $P_i$.
We slightly abuse the definition of a superpixel as used in \lime, where it refers to pixels which are clustered via some segmentation algorithm. In our context, however, it is simply some subset of
pixels of a given image which have not been masked. \rex does not rely on any externally provided initial segmentation. 
The algorithm instead partitions an input image into a small number of rectilinear superpixels, computes their degree of sufficient responsibility 
for the output, and then \emph{iteratively refines} the superpixels with responsibility exceeding some predefined threshold. 

The computation of the degree of sufficient responsibility of each superpixel for the output is precise and is described in \Cref{sec:responsibility}. 
The approximation comes from
the iterative refinement step in the algorithm, described in \Cref{sec:refine_responsibility}, and is due to the fact that the final ranking of
pixels depends on the selected partition. To ameliorate the effect of a particular partition, the algorithm is repeated a number of times
with partitions selected independently at random, and the results are averaged across all the partitions. \Cref{sec:compositional_explanation}
describes the averaging process and the greedy construction of an approximate sufficient explanation from the ranked list of pixels.
Section~\ref{sec:exp0} gives a step-by-step example to illustrate the working of the algorithm.

The scalability of the approach relies on the following observation. 
\begin{observation}~\label{obs:highest}
The pixels with the highest responsibility for the model's decision are located in superpixels
with the highest responsibility.
\end{observation}

Intuitively, the observation holds when pixels with high responsibility do
not appear in the superpixels surrounded by other pixels with very
low responsibility for the input image classification outcome.
Although this can happen in principle, especially in the case of adversarial images, we do not encounter this in practice. While the model has access to data only, this data has strong \emph{correlations} due to the underlying data production mechanism.

\subsection{Computing the Degree of Responsibility of a Superpixel}
\label{sec:responsibility}

Given a set of pixels $\mathcal{P}$, we use $\mathds{P}_i$ to denote a \emph{partition} of $\mathcal{P}$, that is, a 
set $\{P_{i,j}: \bigcup P_{i,j} = \mathcal{P}$ and $\forall j\not=k, P_{i,j} \cap P_{i,k}=\emptyset \}$.
The number of elements in $\mathds{P}_i$ is a parameter, denoted by $s$; in our implementation, \rex, we set $s=4$ as this provides convenient partitioning.
We refer to $P_{i,j}$ as \emph{superpixels}. It is insufficient to consider each superpixel in isolation: there is no reason why any one superpixel should be the cause of a classification. Therefore we create combinations of superpixels from the power set of~$\mathds{P}_i$, denoted $2^{\mathds{P}_i}$.

For an NN $\mathcal{N}$, an input $x$, and a partition $\mathds{P}_i$, we can generalize \Cref{def:simple-resp}
to the set of \emph{superpixels} defined by $\mathds{P}_i$.
We denote by $r_i(P_{i,j},x,\mathcal{N}(x))$ the \emph{sufficient responsibility} of
a superpixel $P_{i,j}$ for $\mathcal{N}$'s classification of $x$, given $\mathds{P}_i$.

For a partition $\mathds{P}_i$, we denote by $X_i$ the set of \emph{mutant images} obtained from $x$ by masking subsets of $2^{\mathds{P}_i}$, and by $\tilde{X}_i$ the subset of $X_i$ that is classified as the original image~$x$. Formally,
\[ \tilde{X}_i = \{ x_m : \mathcal{N}(x_m) = \mathcal{N}(x) \}. \] 

We compute $r_i(P_{i,j},x,\mathcal{N}(x))$, which is an approximation of the degree of sufficient responsibility (\Cref{def:simple-resp})
of each superpixel $P_{i,j}$ for the classification of $x$, in~\Cref{algo:responsibility}.
For a superpixel $P_{i,j}$, we define the set 
\[ \tilde{X}_i^j = \{ x_m : P_{i,j} \mbox{ is not masked in } x_m \} \cap \tilde{X}_i. \]
For a mutant image $x_m$, we define $\mathit{diff_i}(x_m,x)$ as the number of superpixels in the partition $\mathds{P}_i$ that
are masked in $x_m$ (that is, the difference between $x$ and $x_m$ with respect to $\mathds{P}_i$).
For an image $y$, we denote by $y(P_{i,j})$ an image that is obtained by
masking the superpixel $P_{i,j}$ in $y$. 
The responsibility of a superpixel $P_{i,j}$ is calculated by~\Cref{algo:responsibility}
as a minimum difference between a mutant image and the original image over all mutant images~$x_m$ that do not mask $P_{i,j}$, are classified the same as the original image~$x$, and masking $P_{i,j}$ in $x_m$ changes the classification. 


\begin{algorithm}[t]
  \caption{$\mathit{responsibility}(x, \mathds{P}_{i})$}
  \label{algo:responsibility}
  \begin{flushleft}
    \textbf{INPUT:} an image $x$, a partition $\mathds{P}_i$ \\
    \textbf{OUTPUT:} a responsibility map $R: \mathds{P}_i \rightarrow \mathds{Q}$
  \end{flushleft}
  \begin{algorithmic}[1]
    \FOR{each $P_{i,j}\in \mathds{P}_i$}
        \STATE $k \leftarrow \min\limits_{x_m} \{ \mathit{diff}(x_m, x) \,| \,\, x_m \in \tilde{X}^j_i \}$
        \STATE $r_{i,j} \leftarrow \frac{1}{k+1}$
    \ENDFOR    
    \RETURN  $r_{i,0},\dots,r_{i,|P_i|-1}$
  \end{algorithmic}
\end{algorithm}

\subsection{Iterative Refinement of Responsibility}
\label{sec:refine_responsibility}

\Cref{algo:responsibility} calculates the responsibility of each superpixel, subject to a given partition.
It then proceeds with only the high-responsibility superpixels. 
It returns a responsibility map, $R$, which maps superpixels to the rational numbers $\mathds{Q}$, indicating their degree of responsibility.

Note that in general, it is possible that all superpixels in a given partition have the same responsibility.
Consider, for example, a situation where the explanation is right in the middle of the image, and our partition divides the image into four quadrants. Each quadrant would be equally important for the classification, hence we would not gain any insight into why the image was classified in that particular way. In this case, the algorithm starts again from another partition. It can also be the case that there might be multiple disjoint combinations of high responsibility superpixels in~$\mathds{P}_i$. Refining them all can be computationally expensive. \rex allows the user to choose from various pruning strategies to reduce the amount of work performed.

Different combinations of superpixels may have equal explanatory power.
In the event that one superpixel is all that is required, therefore having responsibility 1, the further iterative subdivision of that superpixel is easy. However, if responsibility is split over multiple superpixels, we need a more refined approach. Take, for example, a situation where the left hand side of the image contains the explanation, as can be seen in~\Cref{fig:bison-map}. Here responsibility is split over two super pixels. Let us say that the passing combination of superpixels $P_c = \{0,2\}$, assuming that we have numbered $P_{i,j}$ consecutively, and that $s=4$. Each superpixel $P_{i,j}$ has responsibility $0.5$. We cannot simply mask superpixel $2$ while refining superpixel $0$ as both together are required for the classification. We handle this problem by holding one superpixel to its original value while refining the other and then reversing the procedure. The superpixel held to its original value, but not being refined, does not get any additional responsibility. 

\begin{algorithm}[!htp]
  \caption{$\mathit{iterative\_responsibility\_refinement}(x, \mathds{P}_i)$}
  \label{algo:compositonal_responsibility}
  \begin{flushleft}
    \textbf{INPUT:}\,\, an image $x$ and a partition $\mathds{P}_i$\\
    \textbf{OUTPUT:}\,\, a responsibility map $R: \mathds{P}_i \longrightarrow \mathds{Q}$
  \end{flushleft}
  \begin{algorithmic}[1]
    \STATE $R \leftarrow \mathit{responsibility}(x, \mathds{P}_i)$ 
    \IF {$R$ meets termination condition}
        \RETURN $R$
    \ENDIF
    \STATE $R' \leftarrow \emptyset$
    \FOR{each $P_{c} \in (2^{\mathds{P}_i} - \emptyset)$~s.t~$R(P_c)\not= 0$}
        \STATE $R' \leftarrow R' \, \cup\, \mathit{iterative\_responsibility\_refinement}(x, P_{c})$
    \ENDFOR
    \RETURN $R'$
  \end{algorithmic}
\end{algorithm}

One partition is rarely sufficient for a high-quality (\ie small) explanation, therefore we compute~\Cref{algo:responsibility} many times and compose the results, as shown in~\Cref{algo:compositonal_responsibility}. This calculates the responsibility for each
superpixel (Line~1). If the termination condition is met (Lines~\mbox{2--3}), the responsibility map $Q$ is updated accordingly. Otherwise, for each superpixel in $\mathds{P}_i$, 
we refine it and call the algorithm recursively.
We use $\cup$ to include these newly computed values in the returned map.
The algorithm terminates when: 
1)~the superpixels in $\mathds{P}_i$ are sufficiently refined (containing only very few pixels), or
2)~when all superpixels in $\mathds{P}_i$ have the same responsibility (this condition is for efficiency). In particular, if no further subdivision of a superpixel results in the desired classification, responsibility for all superpixels is $0$ and the algorithm terminates.

\subsection{Explanation Extraction}\label{sec:compositional_explanation}

So far, we assume one particular partition $\mathds{P}_i$, which~\Cref{algo:compositonal_responsibility} recursively refines and calculates the corresponding
responsibilities of superpixels in each step by calling~\Cref{algo:responsibility}.
We note that the choice of the initial partition over the image can affect the values calculated
by~\Cref{algo:compositonal_responsibility}, as this partition determines the set of possible mutants
in~\Cref{algo:responsibility}. We ameliorate the influence of the choice of any particular partition by iterating the algorithm over a set of initial partitions.
Twenty iterations of the algorithm will therefore yield $20$ starting partitions, chosen at random. \rex allows the user to choose from a number of different types of random to build these partitions, with uniform being the default. 

In~\Cref{algo:compositonal_explanation}, we consider $N$ initial partitions and compute an average
of the degrees of responsibility induced by each of these partitions.
In the algorithm, $\mathds{P}^x$ stands for a specific partition chosen randomly from the set of partitions, and $r_p$ denotes the degree of responsibility of a pixel $p$ w.r.t.~$\mathds{P}^x$.

\begin{algorithm}[t]
  \caption{$\mathit{explanation}(x)$}
  \label{algo:compositonal_explanation}
  \begin{flushleft}
    \textbf{INPUT:}\,\, an input image $x$, a parameter $N \in \mathds{N}$\\
    \textbf{OUTPUT:}\,\, an explanation $\mathcal{E}$ 
  \end{flushleft}
  \begin{algorithmic}[1]
    \STATE $r_p \leftarrow 0$ for all pixels $p$
    \FOR{$c$ in $1$ to $N$}
        \STATE $\mathds{P}^x\leftarrow$ sample a partition
        \STATE $R \leftarrow \mathit{iterative\_responsibility\_refinement}(x, \mathds{P}^x)$
        \FOR{each $P_{i,j} \in~\mbox{domain of}~R$}
            \STATE $\forall p\in P_{i,j}: r_p \leftarrow r_p + \frac{R(P_{i,j})}{|P_{i,j}|}$
        \ENDFOR
    \ENDFOR
    \STATE $\mathit{pixel\_ranking} \leftarrow$ pixels from high $r_p$ to low
    \STATE $\mathcal{E}\leftarrow\emptyset$
    \FOR{each pixel $p_i \in \mathit{pixel\_ranking}$}
        \STATE $\mathcal{E}\leftarrow\mathcal{E}\cup\{p_i\}$
        \STATE $x^\mathit{exp}\leftarrow$ mask pixels of $x$ that are \textbf{not} in $\mathcal{E}$
        \IF{$\mathcal{N}(x^\mathit{exp})=\mathcal{N}(x)$} 
          \RETURN{$\mathcal{E}$}
        \ENDIF
    \ENDFOR
  \end{algorithmic}
\end{algorithm}

\Cref{algo:compositonal_explanation} has two parts:
ranking all pixels (Lines 1--9) and constructing the explanation (Lines 10--17). 
The algorithm ranks the pixels of the image according to their responsibility for 
the model's output. 
Each time a partition is randomly selected (Line~3), the
iterative responsibility refinement (\Cref{algo:compositonal_responsibility}) is called to refine it into
a set of fine-grained superpixels and calculate their responsibilities (Line~4). A superpixel's responsibility is evenly distributed to all its pixels, and the pixel-level responsibility is updated accordingly for each sampled partition (Lines~5--7). After $N$ iterations, all pixels are ranked according to their responsibility $r_p$.

The remainder of~\Cref{algo:compositonal_explanation} follows the method for explaining the result of an image classifier by~\citet{sun2020explaining}. That is, we construct a subset of pixels
$\mathcal{E}$ to explain $\mathcal{N}$'s output on this particular input~$x$ \emph{greedily}.
We add pixels to $\mathcal{E}$ as long as $\mathcal{N}$'s output on $\mathcal{E}$ does not
match $\mathcal{N}(x)$. This process terminates when $\mathcal{N}$'s output is the same
as on the whole image $x$. The set $\mathcal{E}$ is returned as an explanation. We note that this extraction procedure is not normally followed by other familiar \xai tools, as their several different definitions of explanation do not extend to finding minimal sets of pixels which induce the desired classification.


While we approximate the computation of an explanation in order to ensure efficiency of our approach, the algorithm is built on solid theoretical foundations, which distinguishes it from other random or heuristic-based approaches. In practice, while our algorithm uses an iterative average of a greedy approximation, it yields highly accurate results (\Cref{sec:evaluation}). Furthermore, our approach is simple and general, and uses the neural network as a black-box. 

\subsection{Illustrative Example}
\label{sec:exp0}

\begin{figure}[t]
    \centering
    \begin{tikzpicture}
        \node[] (bus) {\includegraphics[width=3cm]{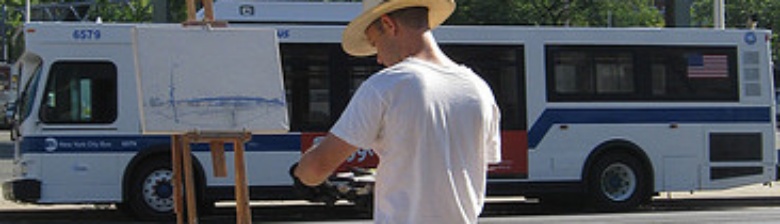}};
      
        \node[above right= -0.5cm and 0.6cm of bus, draw=red, inner sep=1pt] (bus_02) {\includegraphics[width=1.5cm]{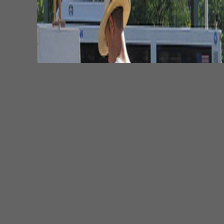}};

        \node[above= 0.25cm of bus_02, draw=red, inner sep=1pt] (bus_01) {\includegraphics[width=1.5cm]{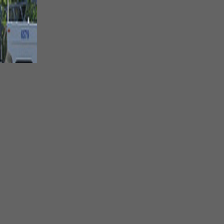}};
        \node[above= 1.5cm of bus_01,  label={[align=center]below:\ding{172} initial\\partitions}] (init){};
        \node[below= 0.125cm of bus_02, draw=green, inner sep=1pt] (bus_03) {\includegraphics[width=1.5cm]{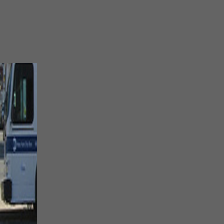}};
        \node[below= 0.25cm of bus_03, draw=green, inner sep=1pt] (bus_04) {\includegraphics[width=1.5cm]{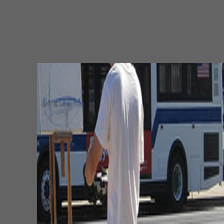}};

        \draw[->, thick] (bus) -- (bus_01);
        \draw[->, thick] (bus) -- (bus_02);
        \draw[->, thick] (bus) -- (bus_03);
        \draw[->, thick] (bus) -- (bus_04);

        \node[above right= -0.8cm and 0.8cm of bus_01, draw=red, inner sep=1pt] (m1) {\includegraphics[width=1.2cm]{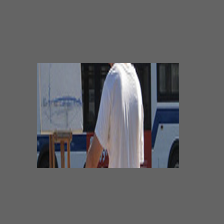}};
        \node[above= 1cm of m1, label={[align=center]below:\ding{173} refined\\partitions}] (ref){};
        \node[below= 0.02cm of m1, draw=green, inner sep=1pt] (m2) {\includegraphics[width=1.2cm]{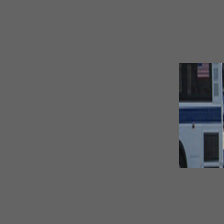}};
        \node[below= 0.02cm of m2, draw=red, inner sep=1pt] (m3) {\includegraphics[width=1.2cm]{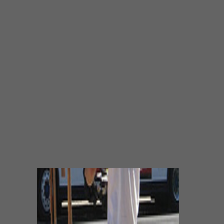}};
        \node[below= 0.02cm of m3, draw=green, inner sep=1pt] (m4) {\includegraphics[width=1.2cm]{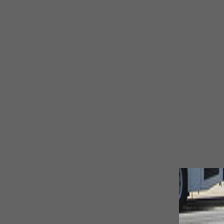}};
        \node[below= 0.02cm of m4, draw=red, inner sep=1pt] (m5) {\includegraphics[width=1.2cm]{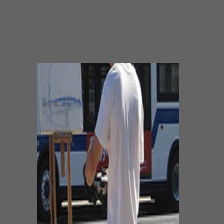}};
        \node[below= 0.02cm of m5, draw=green, inner sep=1pt] (m6) {\includegraphics[width=1.2cm]{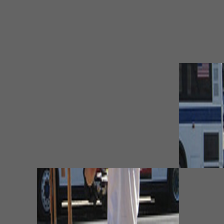}};

        \begin{scope}[on background layer]
            \draw[rounded corners, fill=blue!10] (2, 5) rectangle ++(4.2, -9);
        \end{scope}

        \node[right= 4.7cm of bus, label={[align=center]above:\ding{174} responsibility\\map}] (map) {\includegraphics[scale=0.14]{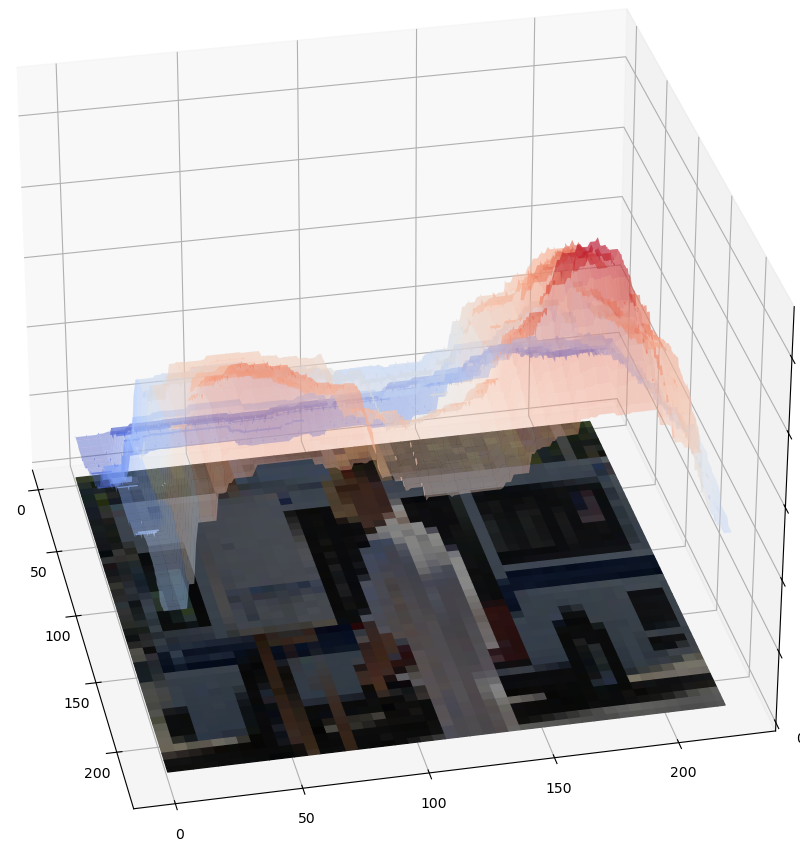}};
         \node[right= 0.5cm of map, label={[align=center]above:\ding{175} extracted\\explanation}] (exp) {\includegraphics[scale=0.25]{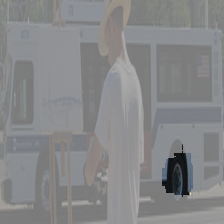}};

        \draw[->, thick] (m1) -- (map);
        \draw[->, thick] (m2) -- (map);
        \draw[->, thick] (m3) -- (map);
        \draw[->, thick] (m4) -- (map);
        \draw[->, thick] (m5) -- (map);
        \draw[->, thick] (m6) -- (map);

        \draw[->, thick] (map) -- (exp);
         
    \end{tikzpicture}
    \caption{The \rex algorithm in action: \rex creates an initial random partition of an image into $4$ sections (\ding{172}). All combinations of these sections are queried by the model, with further refinement applied to those sections or combinations of sections that meet the requirements (\ding{173}). Some combinations, highlighted in green, are classified as bus, others, in red, are not.  This results, after several iterations, in a detailed responsibility map (\ding{174}), from which a minimal passing explanations can be extracted (\ding{175}).}
    \label{fig:rex_example}
\end{figure}

To illustrate how \cet works, consider Figure~\ref{fig:rex_example}, which is classified as \lab{bus} by a ResNet50, even though there is an occlusion in the middle. Initially, \cet picks an arbitrary partition of the image. This results in $15$ combinations of masking superpixels, though in practice, we only need to examine $14$ combinations as the $15^{th}$ is the entire image, which we have already tested. We use~\Cref{algo:responsibility} to calculate the responsibility of each superpixel. Those superpixels or combinations of superpixels which contribute towards the correct classification are further refined. This iterative refinement is a recursive application of the same initial random partitioning into four superpixels, as in~\Cref{algo:compositonal_responsibility}.

The responsibility maps shown in~\Cref{fig:demo-heatmaps} demonstrate the importance of~\Cref{algo:compositonal_explanation}. The initial partitioning of the image can greatly affect the quality of the responsibility map.~\Cref{fig:busn1} shows the responsibility map after one iteration of the algorithm. While it still indicates well the areas of interest in the image it is also rather crude.~\Cref{fig:busn10} contains more refined information about the distribution of responsibility over the pixels in the image. We now have two distinct peaks of responsibility on either side of the central occlusion (a person). Further iterations bring about further refinement. By iteration $100$ (\Cref{fig:busn100}), it is clear that the `dominant' explanation is around the rear tire of the bus. The other explanation however, while depressed in relation to the tire, has not vanished and can still be found through searching the responsibility map~\citep{CKK25}. Moreover, a disjoint explanation exists for this image, as given in~\Cref{fig:bus_occ}.

\begin{figure}[t]
    \centering
    \begin{subfigure}{0.18\linewidth}
        \centering
        \includegraphics[scale=0.25]{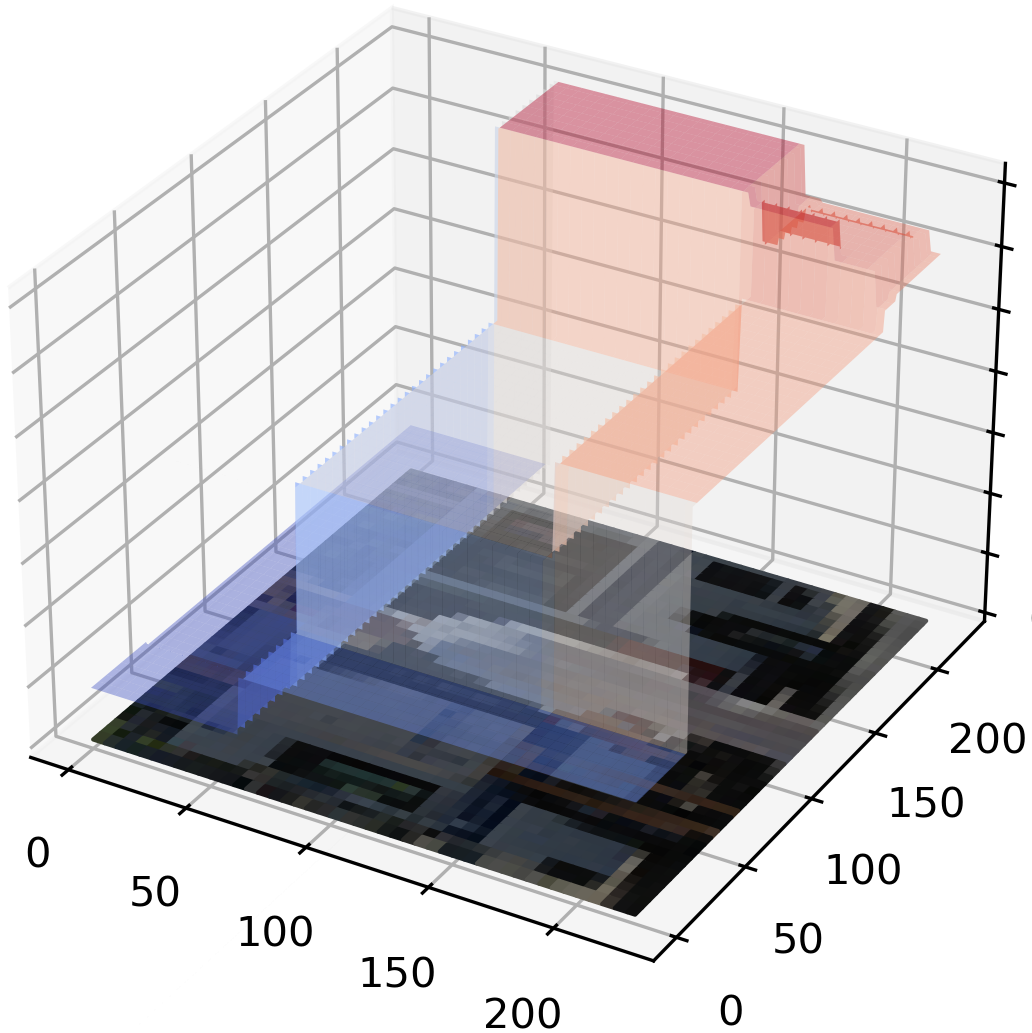}
        \caption{$N = 1$}
        \label{fig:busn1}
    \end{subfigure}
     \begin{subfigure}{0.18\linewidth}
        \centering
        \includegraphics[scale=0.25]{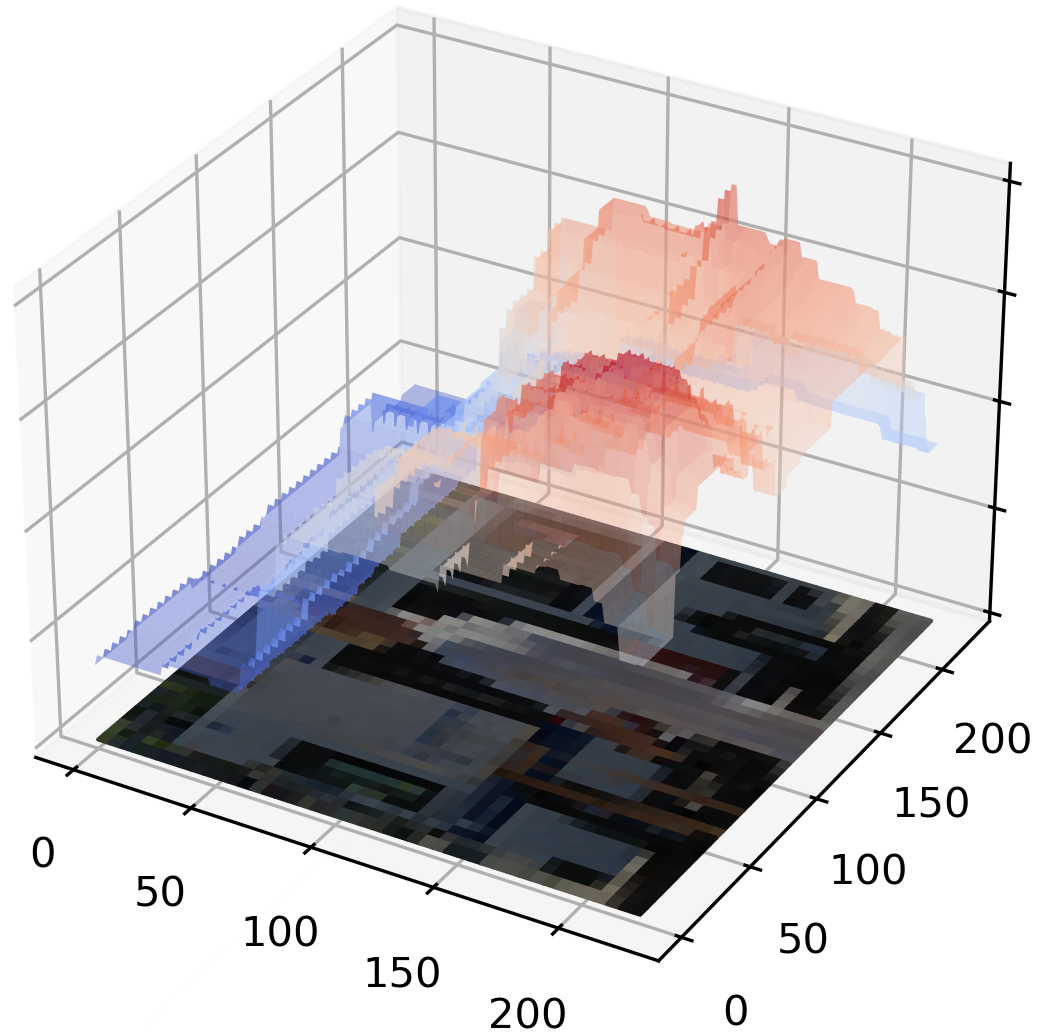}
        \caption{$N = 10$}
        \label{fig:busn10}
    \end{subfigure}
     \begin{subfigure}{0.18\linewidth}
        \centering
        \includegraphics[scale=0.25]{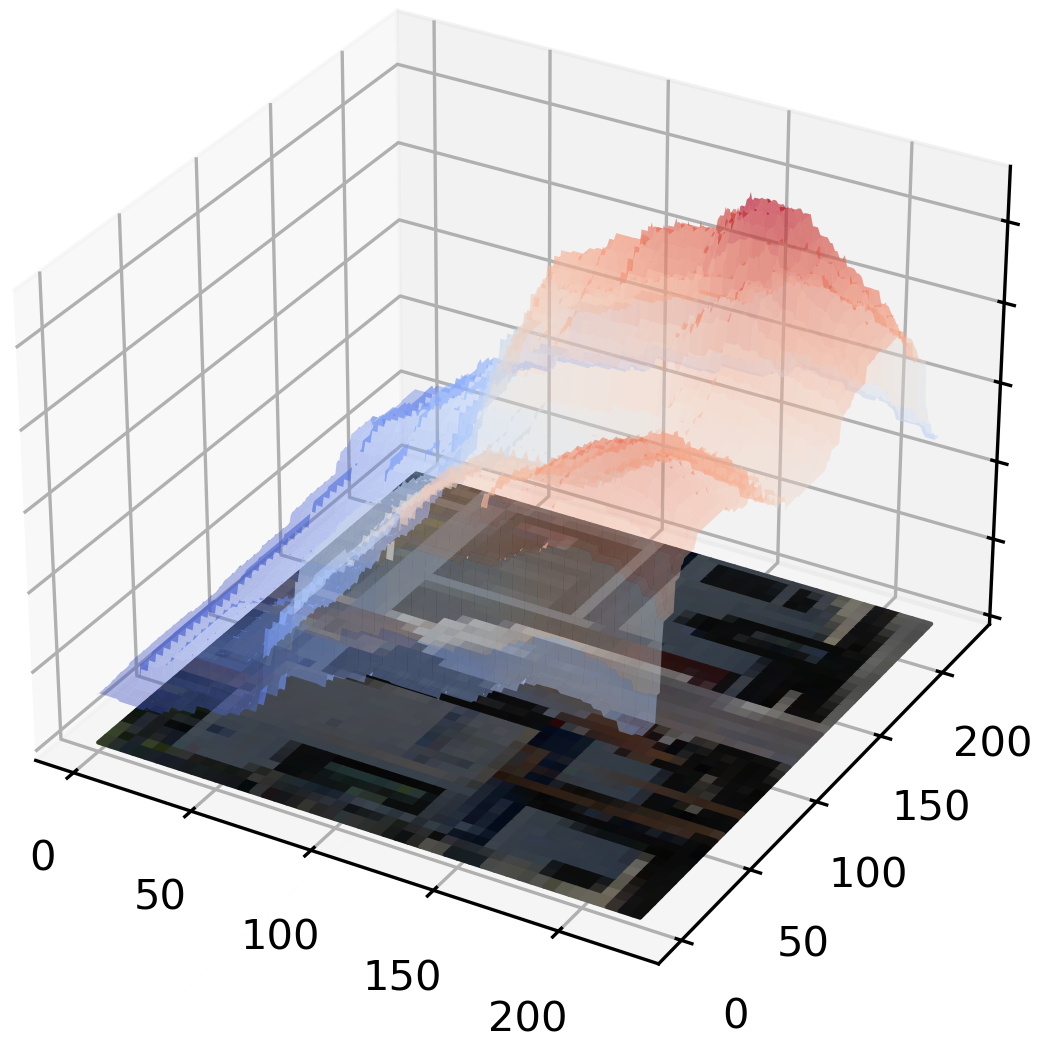}
        \caption{$N = 30$}
        \label{fig:busn30}
    \end{subfigure}
     \begin{subfigure}{0.18\linewidth}
        \centering
        \includegraphics[scale=0.25]{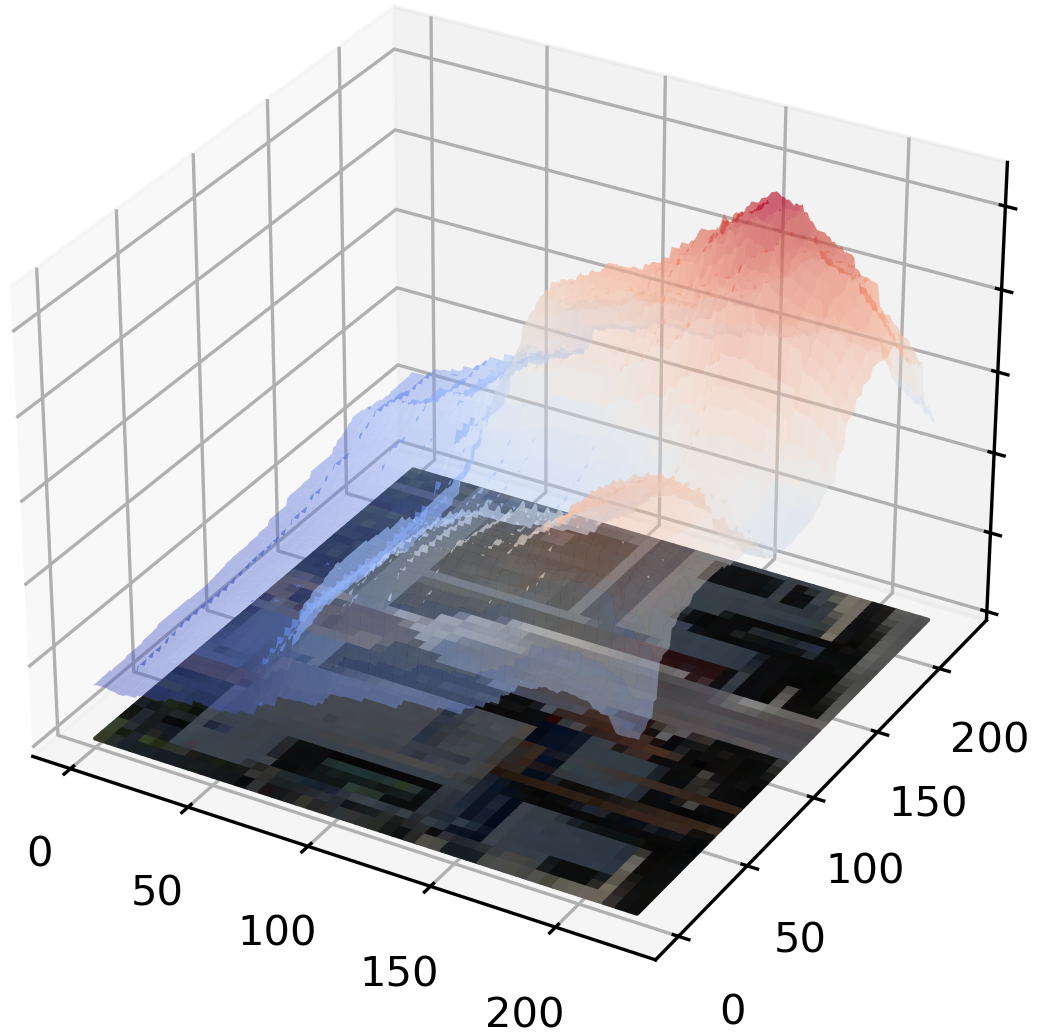}
        \caption{$N = 60$}
        \label{fig:busn60}
    \end{subfigure}
     \begin{subfigure}{0.18\linewidth}
        \centering
        \includegraphics[scale=0.25]{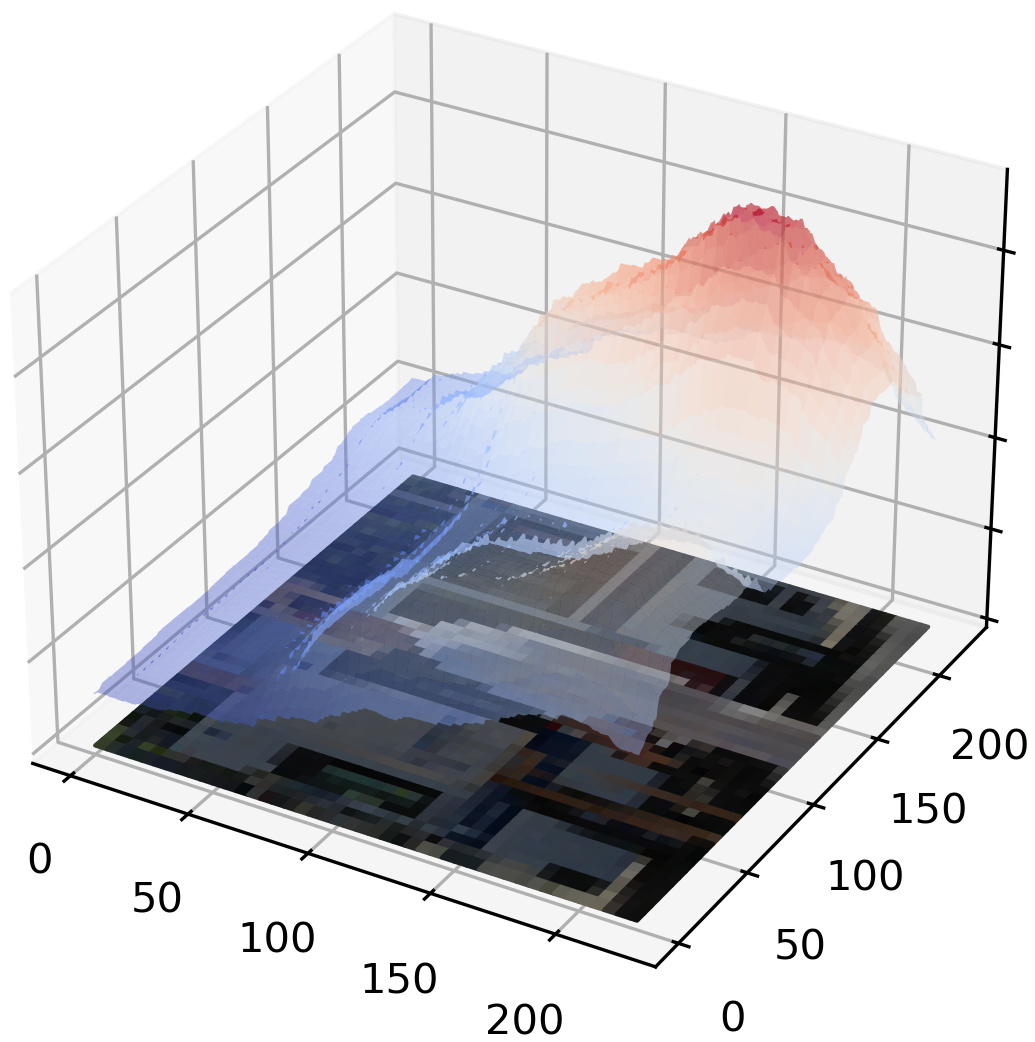}
        \caption{$N = 100$}
        \label{fig:busn100}
    \end{subfigure}
    \caption{Improvement of \cet's pixel ranking on `bus' (\Cref{fig:rex_example}) as the number of iterations $N$ increases (\Cref{algo:compositonal_explanation})}
    \label{fig:demo-heatmaps}
\end{figure}

\begin{figure}
    \centering
    \begin{subfigure}{0.48\linewidth}
        \includegraphics[scale=0.25]{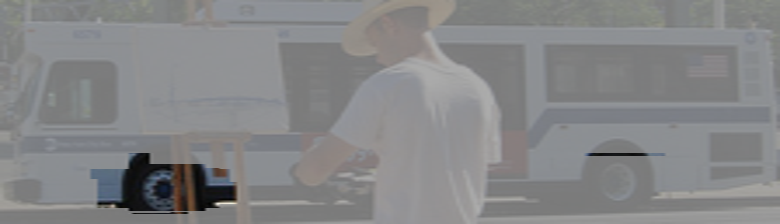}
        \caption{Non-contiguous explanation}
        \label{fig:bus_occ}
    \end{subfigure}
    \hfill
    \begin{subfigure}{0.48\linewidth}
        \includegraphics[scale=0.25]{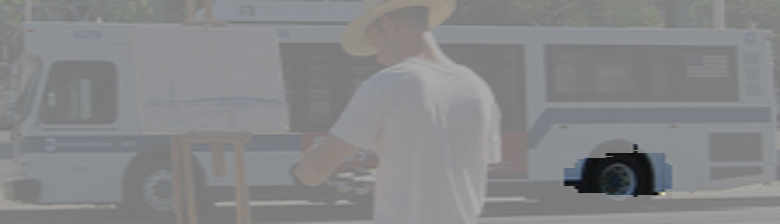}
        \caption{Contiguous explanation}
        \label{fig:bus_min}
    \end{subfigure}
    \caption{Two different explanations extracted from the $30$ iteration ranking in~\Cref{fig:demo-heatmaps}.~\Cref{fig:bus_occ} contrasts interestingly with~\Cref{fig:bus_min}: it seems that the front tire
    by itself is not sufficient (a very small rectangular section of the rear tire is included in the explanation), but the rear tire alone is sufficienct.} 
    \label{fig:demo-exp}
\end{figure}

\subsection{Termination and Complexity}

It is clear that the algorithm necessarily terminates as soon as the parts can no longer be divided into superpixels. Moreover, the number of calls it performs to the
model is linear in the size of the image, as proved in the following lemma.
\begin{lemma}\label{ce-complexity}
The number of calls of ~\Cref{algo:compositonal_explanation} to the model is $O(2^s n N)$, where $s$ is the size of the partition in each step
(in our setting $s=4$), $n$~is the number of pixels in the original image $x$, and $N$ is the number of initial partitions.
\end{lemma}
\begin{proof}
The computation of responsibilities of superpixels in one partition is $O(2^s)$,
as the algorithm examines the effect of mutating each
subset of the superpixels in the current partition. Note that $s$ is a constant independent of the size of the image.
The number of steps is determined by the termination condition on the size of a single superpixel, which in the worst case is the same as a single pixel. 
In our setting, the algorithm terminates when a single superpixel is contains fewer than $10$ pixels\footnote{\rex also terminates when a user-defined work budget is exceeded.}. However, in general, the algorithm can continue down to the level of a single pixel, thus resulting in $n$ pixels in the last step.
The algorithm performs $N$ iterations, and every iteration uses a different initial partition. The parameter $N$ is independent of the size of the image.
\end{proof}

Recall the research question \textbf{RQ3}:
\begin{description}
    \item[RQ3] Is there an efficiently computable approximation to explanations? 
\end{description}

The complexity analysis of \Cref{algo:compositonal_explanation} in \Cref{ce-complexity} provides an affirmative answer
to this research question.

\section{Evaluation}
\label{sec:evaluation}

We conduct a large scale investigation into \rex, comparing it with a host of popular \xai tools. Our answers to the research questions will largely remain empirical.
Just as there is no universally agreed definition of explanation, there is also no single best way to evaluate the quality of an explanation. 

\begin{description}
    \item[RQ4] What are suitable quality measures for explanations?
\end{description}

A causal explanation is a minimal set of pixels from an image which are sufficient to obtain the same top-$1$ classification as the complete set of all pixels of the image. This immediately suggests size of explanation as a robust quality measure. As none of the tools (apart from \rex) compute minimal, sufficient explanations, we use our mechanism of extraction (\Cref{sec:compositional_explanation}) on the output of all tools tested. 

We also use a number of complementary measures as proxies for the quality of the explanations. Insertion and deletion curves, introduced in~\citet{petsiuk2018rise}, are widely used to assess the quality of saliency maps. 
For insertion curves, the model confidence on a target class is measured as pixels are inserted over a baseline value. 
The order of insertion is derived from the map. If a map has accurately identified the most important pixels for a class, then the class confidence should rise quickly. 
This results in a large AUC (area under curve). Deletion curves are calculated in the opposite way, replacing the pixels with a baseline value.
Again, if the map has identified the most important pixels accurately, then the model confidence should drop quickly, resulting in a low AUC.

While a \rex explanation is not a map, we calculate insertion and deletion over the responsibility map. The AUC for both insertion and deletion are obviously tied to the original confidence of the model. To allow for a fair comparison between images with different initial confidences, we follow~\citet{calderon2024real} and normalize all curves by the initial confidence of the model on the image under test. Note that this can result in the the insertion curve AUC being greater than~$1$. This phenomenon occurs when the confidence on the entire image is lower than the confidence of intermediate insertion stages. This makes intuitive sense: if the pixel ranking is accurate, then those pixels that either do not contribute towards the classification, or reduce confidence, are left to be added last. Without these `negative' pixels, confidence is potentially much higher. 

\paragraph*{Comparison with human segmentation} For those datasets that contain segmentation information, we also measure the intersection of the minimal explanation with that segmentation. This segmentation is human-provided and we argue corresponds most closely with what a human considers acceptable.
On VOC2012, for example, we would like a minimal explanation to reside mostly inside the human-provided segmentation mask and therefore contain only a small number of pixels outside the mask. For datasets that feature occluded objects, a good explanation would have as few pixels from inside the occlusion region as possible, as the occlusion should not make a substantial contribution towards the model's classification.

To answer \textbf{RQ4}, we argue that the size of explanations, as well as insertion curves and the intersection with the human segmentation are measures that match our intuition. However, we argue that deletion curves are not, as they remain relatively high in presence of multiple detected explanations (\Cref{subsec:discussion}). 

\subsection{Experimental Setup}\label{sec-exp-setup}

We answer research questions {\bf RQ5--RQ7} experimentally.
We have implemented the proposed explanation approach in the publicly available tool \rex. For all other \xai tools, apart from \rise, we used the Captum \pytorch library~\citep{kokhlikyan2020captum}\footnote{\url{https://www.captum.ai}}. For \rise, we used the authors' implementation\footnote{\url{https://github.com/eclique/RISE}}, which we lightly altered to use \pytorch rather than the original TensorFlow framework. 

In the evaluation, we compare \rex with a wide range of explanation tools, specifically \gradcam~\citep{CAM}, \ks~\citep{lundberg2017unified}, \gs~\citep{gradientshap}, \rise~\citep{petsiuk2018rise}, \lime~\citep{lime}, \ig~\citep{sundararajan2017axiomatic} and~\noise~\citep{smilkov2017smoothgradremovingnoiseadding}. For the ResNet50 model only, we also include \lrp~\citep{LRP}. Unfortunately, the Captum version of \lrp could not run on the other tested models.

While our main interest is in black-box explainability, we have included methods which require access to the model gradient, such as \gs and \ig, and internal model layers, \ie \gradcam. This is largely due to the relative paucity of purely black-box \xai methods for image data.

We use $4$ data sets: \imagenet-1k-mini validation~\citep{imagenet}, 
ECSSD~\citep{ecssd},
Pascal VOC2012~\citep{pascal-voc-2012},  and a ``Photobombing'' dataset we created. Imagenet1k-mini comes with labels for ground truth. VOC2012 has labels and segmentation data. ECSSD comprises complex images which come with a human-provided segmentation.
We created the ``Photobombing'' dataset by inserting black occlusions into \imagenet images, meaning we have the original label and also know the pixel coordinates of the occlusions. These occlusions are placed so as not to change the model classification. We use the \textsc{torchvision}\footnote{\url{https://pytorch.org/vision/stable/index.html}} implementations of ResNet50, ViT and ConvNext-Large with default weights on all data sets.

All experiments were conducted using a server running Ubuntu 20.04 with an Nvidia A40 GPU. All tools were used with default settings. For \rex in particular, this means that it performs $20$ iterations of the algorithm ($N$ in~\Cref{algo:compositonal_explanation}), with a minimum superpixel size of $10$ pixels. The pruning strategy for the work queue is ``area'' (the passing mutants are ordered by size, smallest to largest) and only one item is kept in the work queue at a time, so we only refine the smallest passing mutant. Partitions are created uniformly at random. Other strategies and partitioning distributions are available in the configuration. 

\rex uses four superpixels per partition. This is practical for images as we can split the image with just one call to a random number generator. The number of mutants produced is also relatively small ($15$), which can usually be fit into a single batch for model inference. Having more initial superpixels does not lead to greater expressivity: we iteratively refine passing combinations of superpixels, essentially recreating the more detailed initial partition which would be produced by a greater number of superpixels. The termination conditions for the partition refinement in~\Cref{algo:compositonal_responsibility} are: 1)~the area of a superpixel is less than $10$ pixels of (resized) input image, or 2)~the four superpixels share the same responsibility.

\subsection{Results}

\begin{description}

    \item[RQ5] What is the precision of explanations computed by our algorithm compared to other \xai methods?
    \item[RQ6] Is there a trade-off between precision and compute cost of the explanations?
    \item[RQ7] Can black-box methods achieve the same quality of explanations as white and grey-box methods?
\end{description}

\begin{table}[t]
    \caption{Results for ResNet50 model.}
    \label{tab:resnet50_eval}
    \centering
    \begin{subtable}{1\textwidth}
        \centering
        \caption{Results on Voc}
        \begin{tabular}{r|r|r|r|r|r|r}
        tool & area ($\downarrow$) & std & ins ($\uparrow$) & del & IN ($\uparrow$) & OUT ($\uparrow$) \\ 
        \midrule
        \midrule
        \gradcam & 0.0832 & 0.1202 & 1.0488 & 0.2756 & 0.3202 & -- \\ 
        \noise & 0.2867 & 0.2627 & 0.6953 & 0.0851 & 0.3241 & -- \\ 
        \lrp  & 0.3936 & 0.2828 & 0.5267 & 0.2953 & 0.2093 & -- \\  
        \ig  & 0.4904 & 0.2768 & 0.4899 & 0.1228 & 0.2215 & -- \\ 
        \gs  & 0.3894 & 0.3002 & 0.5935 & \textbf{0.0651} & 0.2872 & -- \\ 
        \ks  & 0.7741 & 0.2084 & 0.188 & 0.1941 & 0.1636 & -- \\ 
        \lime  & 0.1124 & 0.1616 & 0.963 & 0.2329 & 0.3039 & -- \\ 
        \rex  & \textbf{0.0427} & \textbf{0.0505} & \textbf{1.2218} & 0.2148 & \textbf{0.6432} & -- \\ 
        \rise & 0.1271 & 0.1547 & 0.9358 & 0.2264 & 0.2861 & -- \\ 
        \bottomrule
    \end{tabular}
      \end{subtable}
      \bigskip
      
      \begin{subtable}{1\textwidth}
        \centering
        \caption{Results on Imagenet}
        \begin{tabular}{r|r|r|r|r|r|r}
        tool & area ($\downarrow$) & std & ins ($\uparrow$) & del & IN ($\uparrow$) & OUT ($\uparrow$) \\
        \midrule
        \midrule
        \gradcam & 0.0486 & 0.0839 & 1.0059 & 0.4073 & -- & -- \\ 
        \noise & 0.1847 & 0.2131 & 0.7612 & 0.1274 & -- & -- \\ 
        \lrp & 0.2749 & 0.2418 & 0.5972 & 0.3567 & -- & -- \\ 
        \ig & 0.4063 & 0.264 & 0.534 & 0.1767 & -- & -- \\ 
        \gs & 0.2861 & 0.2696 & 0.6474 & \textbf{0.0983} & -- & -- \\  
        \ks & 0.6686 & 0.2472 & 0.2549 & 0.2649 & -- & -- \\ 
        \lime & 0.0649 & 0.0969 & 0.9871 & 0.3313 & -- & -- \\ 
        \rex & \textbf{0.0333} & \textbf{0.0351} & \textbf{1.1101} & 0.3403 & -- & -- \\ 
        \rise & 0.1026 & 0.1254 & 0.8898 & 0.3728 & -- & -- \\ 
                \bottomrule
      \end{tabular}  
      \end{subtable}
     \bigskip
     
    \begin{subtable}{1\textwidth}
        \centering
        \caption{Results on ``Photobombing''}
        \begin{tabular}{r|r|r|r|r|r|r}
        tool & area ($\downarrow$) & std & ins ($\uparrow$) & del & IN ($\uparrow$) & OUT ($\uparrow$) \\
        \midrule
        \midrule
        \gradcam & 0.0455 & 0.0642 & 0.9843 & 0.3625 & -- & 0.8505 \\ 
        \noise & 0.1565 & 0.1832 & 0.7604 & 0.1348 & -- & 0.8732 \\ 
        \lrp & 0.2429 & 0.205 & 0.6049 & 0.3687 & -- & 0.9058 \\ 
        \ig & 0.4184 & 0.2311 & 0.5178 & 0.1853 & -- & 0.8806 \\ 
        \gs & 0.296 & 0.2373 & 0.637 & \textbf{0.114} & -- & 0.8533 \\ 
        \ks & 0.6601 & 0.2246 & 0.2527 & 0.2612 & -- & 0.9188 \\ 
        \lime & 0.0507 & 0.0523 & 0.9937 & 0.3046 & -- & 0.8921 \\ 
        \rex & \textbf{0.0297} & \textbf{0.026} & \textbf{1.0481} & 0.2947 & -- & \textbf{0.9739} \\ 
        \rise & 0.0787 & 0.0942 & 0.8906 & 0.2977 & -- & 0.8798 \\ 
                \bottomrule
      \end{tabular}  
      \end{subtable}
\end{table}
\begin{table}
\ContinuedFloat
    \centering  
    \phantomcaption
      \begin{subtable}{1\textwidth}
        \centering
                \caption{Results on ECSSD}
        \begin{tabular}{r|r|r|r|r|r|r}
        tool & area ($\downarrow$) & std & ins ($\uparrow$) & del & IN ($\uparrow$) & OUT ($\uparrow$) \\
        \midrule
        \midrule
        \gradcam & 0.0836 & 0.1288 & 1.0226 & 0.2915 & \textbf{0.6654} & -- \\ 
        \noise & 0.2695 & 0.276 & 0.7194 & 0.119 & 0.5955 & -- \\ 
        \lrp & 0.3545 & 0.292 & 0.5721 & 0.305 & 0.3418 & -- \\ 
        \ig & 0.4437 & 0.3016 & 0.5552 & 0.1377 & 0.4187 & -- \\ 
        \gs & 0.3525 & 0.3182 & 0.6444 & \textbf{0.0977} & 0.5379 & -- \\ 
        \ks & 0.7257 & 0.2403 & 0.2348 & 0.2369 & 0.244 & -- \\ 
        \lime & 0.1013 & 0.1579 & 1.0024 & 0.2404 & 0.6538 & -- \\ 
        \rex & \textbf{0.0423} & \textbf{0.0522} & \textbf{1.2176} & 0.2214 & 0.6205 & -- \\ 
        \rise & 0.1165 & 0.1599 & 0.95 & 0.2752 & 0.5451 & -- \\ 
      \end{tabular}  
      \end{subtable}  
\end{table}

\Cref{tab:resnet50_eval,tab:vit_eval,tab:conv_eval} show the experimental results of evaluation of \rex against $8$ other
\xai tools over $4$ datasets, with three different models. 
The columns in the tables are: mean area (area), $\sigma$ of area (standard deviation), 
normalized insertion curve AUC (ins), normalized deletion curve AUC (del), proportion of explanation inside the relevant segment (IN), 
and the fraction of the explanation that is outside the relevant segment (OUT), shown only where appropriate. 
Bold indicates best result in category for all columns. 
Note that while we have indicated lowest AUC for deletion as `best', lower here is not necessarily better, as we argue in~\Cref{subsec:discussion}.

\paragraph*{Resnet} \Cref{tab:resnet50_eval} shows the results with the ResNet50 model. \rex consistently produces the most smallest explanations. This is reinforced by the observations that \rex also produces the highest insertion AUC of any of the tools. \rex does particularly well on the `Photobombing` dataset, where its explanations are almost entirely ($0.97$) disjoint from the inserted occlusion 
(\Cref{fig:photo-bombing-examples}).
\Cref{fig:resnet-box} shows that \rex has the most concentrated explanations and consistent results, with very few outliers. 
The grey-box tools perform worse than the other methods. Note that \rex even outperforms \gradcam, a white-box method.

\begin{table}
    \caption{Results for ViT model.}
    \label{tab:vit_eval}
    \centering
        \begin{subtable}{1\textwidth}
            \centering
            \caption{Results on Voc}
        \begin{tabular}{r|r|r|r|r|r|r}
        tool & area ($\downarrow$) & std & ins ($\uparrow$) & del & IN ($\uparrow$) & OUT ($\uparrow$) \\ 
      \midrule
       \midrule
        \gradcam & 0.7148 & 0.181 & 0.2466 & 0.1889 & 0.1727 & -- \\ 
        \noise & 0.4988 & 0.3019 & 0.4565 & 0.0633 & 0.2315 & -- \\ 
        \ig & 0.5652 & 0.2632 & 0.4274 & 0.0943 & 0.1946 & -- \\ 
        \gs & 0.5021 & 0.2992 & 0.4761 & \textbf{0.0551} & 0.2267 & -- \\ 
        \ks & 0.669 & 0.2603 & 0.2579 & 0.2579 & 0.1686 & -- \\ 
        \lime & 0.277 & 0.2218 & \textbf{0.7191} & 0.112 & 0.26 & -- \\ 
        \rex & \textbf{0.2206} & \textbf{0.1711} & 0.6985 & 0.1377 & \textbf{0.5496} & -- \\ 
        \rise & 0.4049 & 0.2491 & 0.5986 & 0.097 & 0.2365 & -- \\ 
        \bottomrule
    \end{tabular}
      \end{subtable}
      \bigskip

      \begin{subtable}{1\textwidth}
            \centering
            \caption{Results on Imagenet}
        \begin{tabular}{r|r|r|r|r|r|r}
        tool & area ($\downarrow$) & std & ins ($\uparrow$) & del & IN ($\uparrow$) & OUT ($\uparrow$) \\
        \midrule
         \midrule
        \gradcam & 0.6277 & 0.203 & 0.2943 & 0.2538 & -- & -- \\ 
        \noise & 0.3903 & 0.2711 & 0.5341 & \textbf{0.098} & -- & -- \\ 
        \ig & 0.4706 & 0.2505 & 0.4833 & 0.1481 & -- & -- \\ 
        \gs & 0.3871 & 0.2691 & 0.5586 & 0.0916 & -- & -- \\ 
        \ks & 0.5483 & 0.2651 & 0.3541 & 0.3519 & -- & -- \\ 
        \lime & 0.2246 & 0.1937 & \textbf{0.7264} & 0.1668 & -- & -- \\ 
        \rex & \textbf{0.1689} & \textbf{0.133} & 0.687 & 0.2015 & -- & -- \\ 
        \rise & 0.3356 & 0.2333 & 0.5997 & 0.1664 & -- & -- \\ 
        \bottomrule
        \end{tabular}  
      \end{subtable}
     \bigskip

      \begin{subtable}{1\textwidth}
            \centering
            \caption{Results on ``Photobombing''}
        \begin{tabular}{r|r|r|r|r|r|r}
        tool & area ($\downarrow$) & std & ins ($\uparrow$) & del & IN ($\uparrow$) & OUT ($\uparrow$) \\
         \midrule
         \midrule
        \gradcam & 0.6025 & 0.1907 & 0.303 & 0.2715 & -- & 0.9196 \\ 
        \noise & 0.3647 & 0.2447 & 0.5439 & 0.1012 & -- & 0.8835 \\ 
        \ig & 0.4598 & 0.2251 & 0.4806 & 0.1486 & -- & 0.8889 \\ 
        \gs & 0.3726 & 0.2465 & 0.5607 & \textbf{0.0967} & -- & 0.8802 \\ 
        \ks & 0.5168 & 0.233 & 0.37 & 0.3665 & -- & 0.9186 \\ 
        \lime & 0.1748 & 0.1458 & \textbf{0.7365} & 0.1778 & -- & 0.8891 \\ 
        \rex & \textbf{0.1402} & \textbf{0.1046} & 0.6883 & 0.2073 & -- & \textbf{0.9551} \\ 
        \rise & 0.2803 & 0.2056 & 0.6119 & 0.1741 & --& 0.8915 \\ 
      \end{tabular}  
      \end{subtable}
      \bigskip
\end{table}

\begin{table}
\ContinuedFloat
    \centering  
    \phantomcaption
      \begin{subtable}{1\textwidth}
            \centering
            \caption{Results on ECSSD}
        \begin{tabular}{r|r|r|r|r|r|r}
        tool & area ($\downarrow$) & std & ins ($\uparrow$) & del & IN ($\uparrow$) & OUT ($\uparrow$) \\
        \midrule
         \midrule
        \gradcam & 0.6741 & 0.2181 & 0.2731 & 0.2344 & 0.2494 & -- \\ 
        \noise & 0.4298 & 0.3001 & 0.5371 & 0.0731 & 0.4296 & -- \\ 
        \ig & 0.4727 & 0.2723 & 0.5184 & 0.1082 & 0.3409 & -- \\ 
        \gs & 0.4179 & 0.2929 & 0.5772 & \textbf{0.0685} & 0.4134 & -- \\ 
        \ks & 0.6192 & 0.2729 & 0.321 & 0.3188 & 0.2466 & -- \\ 
        \lime & 0.2698 & 0.2319 & \textbf{0.7392} & 0.123 & \textbf{0.5548} & -- \\   
        \rex & \textbf{0.2008} & \textbf{0.1653} & 0.7286 & 0.1475 & 0.5334 & -- \\ 
        \rise & 0.3687 & 0.2585 & 0.6597 & 0.1178 & 0.4468 & -- \\ 
      \end{tabular}  
      \end{subtable}
\end{table}

\paragraph*{ViT} \Cref{tab:vit_eval} shows the results with the `vit\_b\_32' model from \textsc{torchvision}. 
Explanations across all tools are much larger than for the ResNet50, which seems to be a feature of this model. \rex, however, still manages to produce explanations with the lowest number of redundant pixels. In general, \rex, \lime and \gradcam are the best performing tools. 
Even in cases where \lime slightly outperforms \rex (\eg, on ECSSD IN), with $0.73$ of its explanation inside the segmentation against $0.72$, the effect is very small. \Cref{fig:vit-boxplot} shows that \rex again has the most concentrated explanations and consistent results,
with very few outliers. \lime is the closest comparable tool, though it has many more outliers and a higher median value. The significant
difference in explanation size between the ResNet model and ViT model is, perhaps, of independent interest.

\begin{table}[]
    \caption{Results for convnext model.}
    \label{tab:conv_eval}
    \centering
    \begin{subtable}{1\textwidth}
        \centering
        \caption{Results on Voc}
        \begin{tabular}{r|r|r|r|r|r|r}
        tool & area ($\downarrow$) & std & ins ($\uparrow$) & del & IN ($\uparrow$) & OUT ($\uparrow$) \\ 
      \midrule
       \midrule
        \gradcam & 0.0556 & 0.0948 & 0.8635 & 0.2497 & 0.3345 & -- \\ 
        \noise & 0.1972 & 0.2443 & 0.689 & 0.117 & 0.3614 & -- \\ 
        \ig & 0.4152 & 0.2565 & 0.465 & 0.2387 & 0.2025 & -- \\ 
        \gs & 0.2268 & 0.2609 & 0.6411 & \textbf{0.0998} & 0.3327 & -- \\ 
        \ks & 0.5739 & 0.3001 & 0.2926 & 0.2924 & 0.1678 & -- \\ 
        \lime & 0.0689 & 0.0972 & 0.8096 & 0.2908 & 0.3042 & -- \\ 
        \rex & \textbf{0.0361} & \textbf{0.038} & \textbf{0.876} & 0.2057 & \textbf{0.6295} & -- \\ 
        \rise & 0.0759 & 0.1129 & 0.8136 & 0.3005 & 0.2863 & -- \\ 
        \bottomrule
    \end{tabular}
      \end{subtable}
      \bigskip

      \begin{subtable}{1\textwidth}
        \centering
        \caption{Results on Imagenet}
        \begin{tabular}{r|r|r|r|r|r|r}
        tool & area ($\downarrow$) & std & ins ($\uparrow$) & del & IN ($\uparrow$) & OUT ($\uparrow$) \\
        \midrule
         \midrule
        \gradcam & 0.0337 & 0.0553 & 0.8613 & 0.4413 & -- & -- \\ 
        \noise & 0.1116 & 0.1608 & 0.7682 & \textbf{0.1886} & -- & -- \\ 
        \ig & 0.3414 & 0.2168 & 0.4993 & 0.3471 & -- & -- \\ 
        \gs & 0.1394 & 0.183 & 0.706 & 0.1678 & -- & -- \\ 
        \ks & 0.4357 & 0.2855 & 0.4042 & 0.4061 & -- & -- \\ 
        \lime & 0.0493 & 0.0707 & 0.8228 & 0.4747 & -- & -- \\ 
        \rex & \textbf{0.0287} & \textbf{0.0236} & \textbf{0.8788} & 0.3487 & -- & -- \\ 
        \rise & 0.0615 & 0.084 & 0.8058 & 0.5136 & -- & -- \\ 
                \bottomrule
      \end{tabular}  
      \end{subtable}
     \bigskip  
    
    \begin{subtable}{1\textwidth}
            \centering
            \caption{Results on ``Photobombing''}
        \begin{tabular}{r|r|r|r|r|r|r}
        tool & area ($\downarrow$) & std & ins ($\uparrow$) & del & IN ($\uparrow$) & OUT ($\uparrow$) \\
        \midrule
         \midrule
        \gradcam & 0.0426 & 0.0527 & 0.8584 & 0.3688 & -- & 0.8695 \\ 
        \noise & 0.0768 & 0.0925 & 0.7916 & \textbf{0.2077} & -- & 0.8755 \\ 
        \ig & 0.3144 & 0.1731 & 0.5454 & 0.3847 & -- & 0.8783 \\ 
        \gs & 0.1335 & 0.1351 & 0.7095 & 0.2186 & -- & 0.85 \\ 
        \ks & 0.3829 & 0.2405 & 0.4683 & 0.468 & -- & 0.9183 \\ 
        \lime & 0.0513 & 0.0551 & \textbf{0.8643} & 0.4161 & -- & 0.8947 \\ 
        \rex & \textbf{0.0296} & \textbf{0.0236} & 0.8475 & 0.3173 & -- & \textbf{0.9748} \\ 
        \rise & 0.063 & 0.0723 & 0.8129 & 0.4769 & -- & 0.9015 \\ 
                \bottomrule
      \end{tabular}  
      \end{subtable}
      \bigskip
\end{table}
\begin{table}
\ContinuedFloat
\phantomcaption
    \centering             
      \begin{subtable}{1\textwidth}
            \centering
            \caption{Results on ECSSD}
        \begin{tabular}{r|r|r|r|r|r|r}
        tool & area ($\downarrow$) & std & ins ($\uparrow$) & del & IN ($\uparrow$) & OUT ($\uparrow$) \\
        \midrule
         \midrule
        \gradcam & 0.0574 & 0.1023 & 0.8865 & 0.2694 & \textbf{0.6846} & -- \\ 
        \noise & 0.2164 & 0.2734 & 0.6424 & 0.1067 & 0.5913 & -- \\ 
        \ig & 0.421 & 0.2925 & 0.4202 & 0.2102 & 0.363 & -- \\ 
        \gs & 0.2434 & 0.284 & 0.5911 & \textbf{0.0987} & 0.5671 & -- \\ 
        \ks & 0.5616 & 0.3152 & 0.2898 & 0.2977 & 0.2468 &-- \\ 
        \lime & 0.0849 & 0.1377 & 0.7157 & 0.3183 & 0.5841 & -- \\ 
        \rex & \textbf{0.0364} & \textbf{0.0405} & \textbf{0.9012} & 0.2077 & 0.5722 & -- \\ 
        \rise & 0.0823 & 0.1387 & 0.7996 & 0.3073 & 0.5045 & -- \\ 
      \end{tabular}  
      \end{subtable}  
\end{table}

\paragraph*{ConvNext} \Cref{tab:conv_eval} shows the results with the `convnext\_large' model from \textsc{torchvision}. \rex, \lime, \gradcam and \rise are the best performing tools. It is interesting to note that $3$ of these tools fall into our black-box category. \Cref{fig:conv-boxplot} shows that \rex has the best performance again, this time slightly ahead of \gradcam, as \gradcam has many more outliers than \rex and a slightly larger inter-quartile range. Explanation size is comparable with that of the ResNet model, 
revealing the ViT model to be somewhat of an outlier.

\begin{figure}[t]
    \centering
    \includegraphics[scale=0.8]{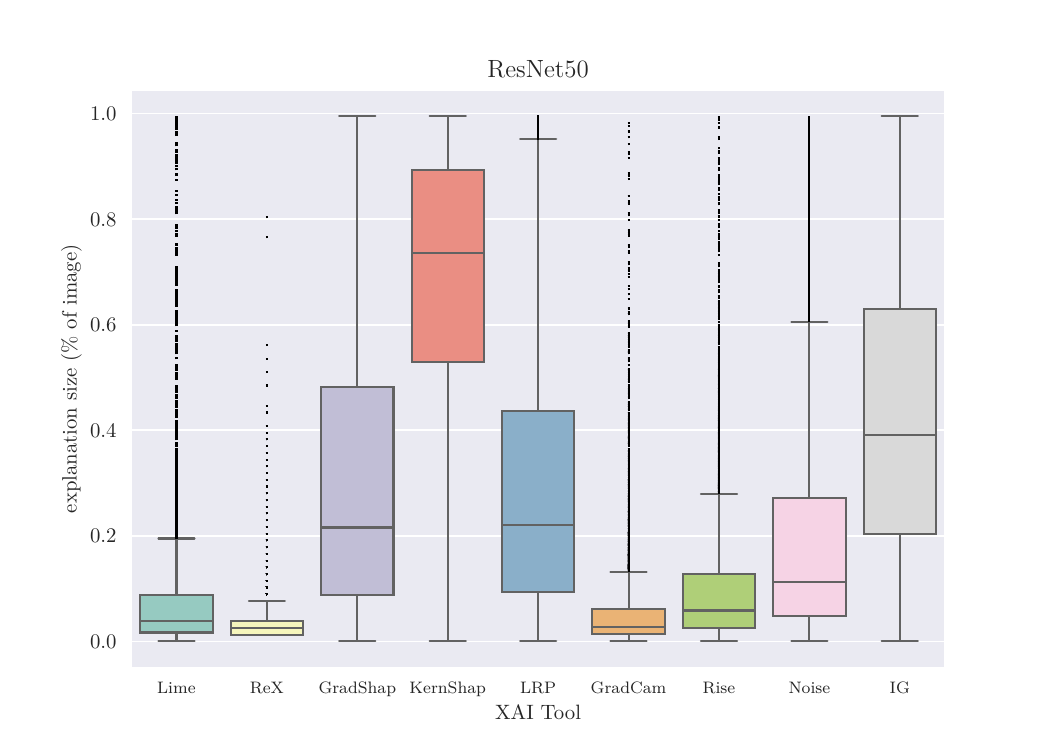}
    \caption{Box plot of all tools over all datasets with a ResNet50 model. \rex has the lowest median value and also the smallest number of outliers.}
    \label{fig:resnet-box}
\end{figure}

\begin{figure}[t]
    \centering
    \includegraphics[scale=0.8]{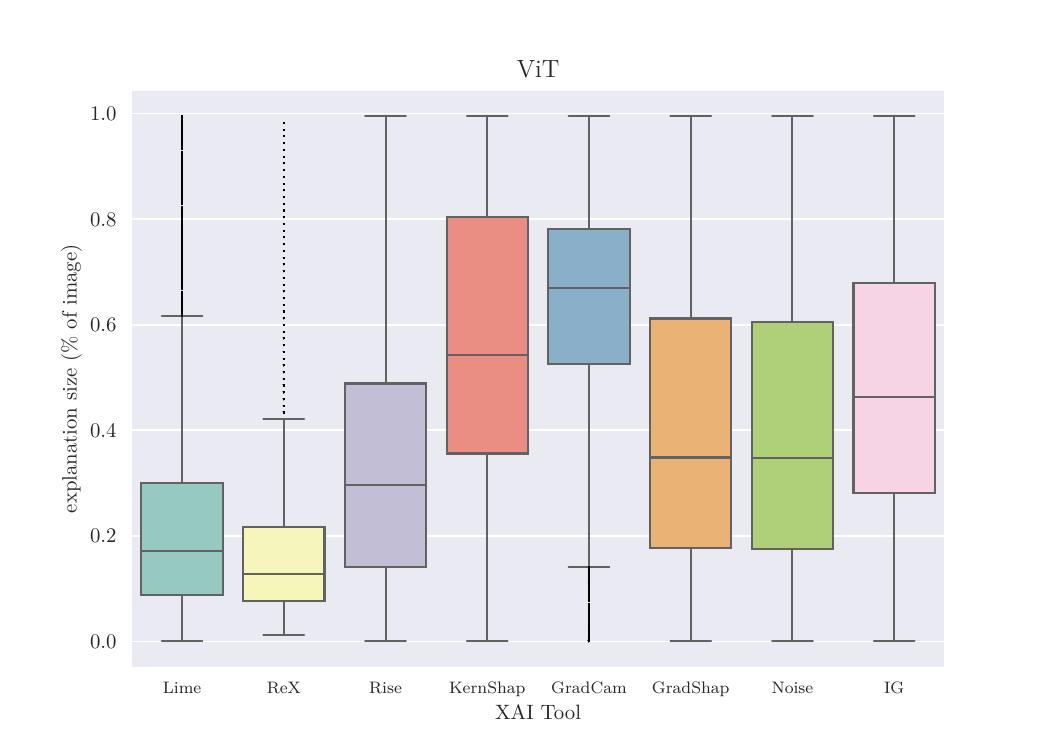}
    \caption{Box plot of all tools over all datasets with a ViT model. \rex has the lowest median value and also the smallest number of outliers. ViT explanations are noticeably larger than those for ResNet50 and ConvNext.}
    \label{fig:vit-boxplot}
\end{figure}

\begin{figure}[t]
    \centering
    \includegraphics[scale=0.8]{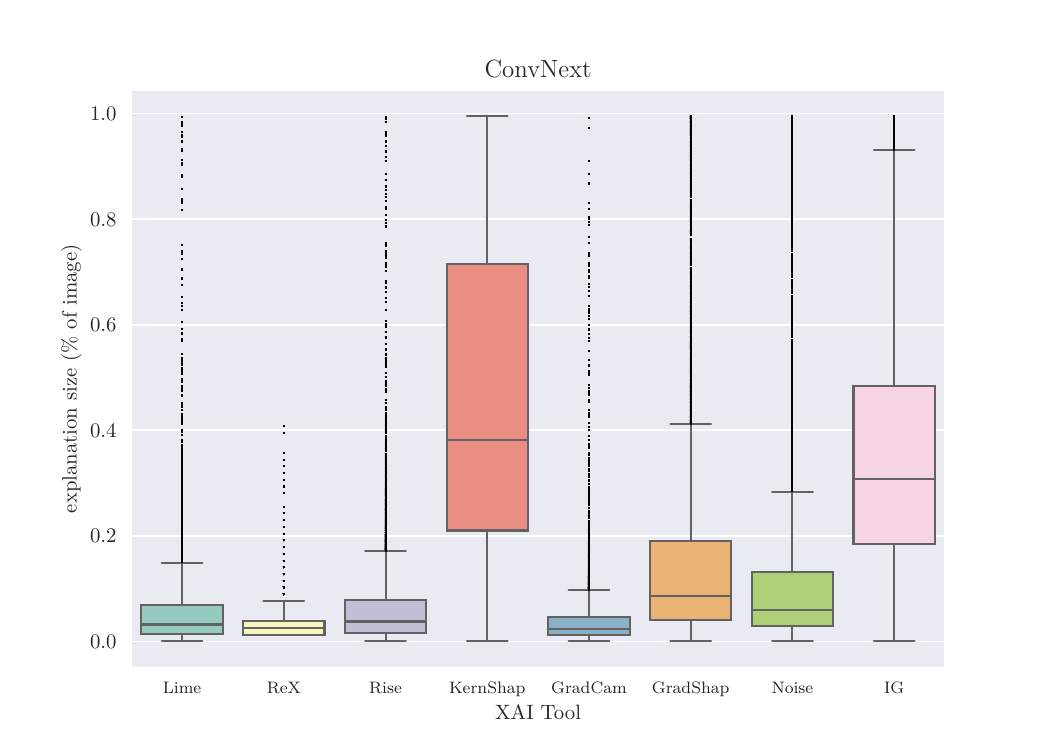}
    \caption{Box plot of all tools over all datasets with a ConvNext model. \rex has the lowest median value and also the smallest number of outliers. It even outperforms \gradcam, a purely white-box tool.}
    \label{fig:conv-boxplot}
\end{figure}

\clearpage
\paragraph*{Runtime} 
We measured the running time of all tools. As all tools were evaluated on the same hardware with the same models and
datasets, the results should be consistent. The white and grey-box methods are clearly the fastest, as expected. Indeed,
\gradcam is essentially instantaneous, as the overhead over one call to the model is insignificant. No black-box tool
can compete with this, by the nature of black-box tools. All the gradient-based methods took $\approx 3$ seconds to
produce initial maps. Note that this is the time taken to generate initial saliency, as minimal pixel subsets still need
to be extracted.
Both \rex and \lime took $\approx 4$s to produce maps, with \rise being much slower ($\approx 10$s). Taking into account the extra time for the explanation extraction (which only \rex does by default), the efficiency picture changes slightly. \rex{}'s total compute time increases only slightly (to $\approx 6$s) due the the uniformly small size of its explanations. Other tools suffer more, even \gradcam{}'s average compute time increases to a total of $\approx 4$s. This is due to the relatively noisy output of \gradcam, requiring more work from the extraction algorithm as a result. \rex is an efficient black-box tool: for \textbf{RQ6} therefore, our results show that there is little trade-off between quality and compute cost. 

\subsection{Discussion}\label{subsec:discussion}

\begin{figure}[t]
    \centering
    \begin{subfigure}{0.45\linewidth}
        \centering
        \includegraphics[scale=0.5]{starfishmap.png}
        \caption{Starfish}\label{fig:starfish-map}
    \end{subfigure}
        \begin{subfigure}{0.45\linewidth}
        \centering
        \includegraphics[scale=0.35]{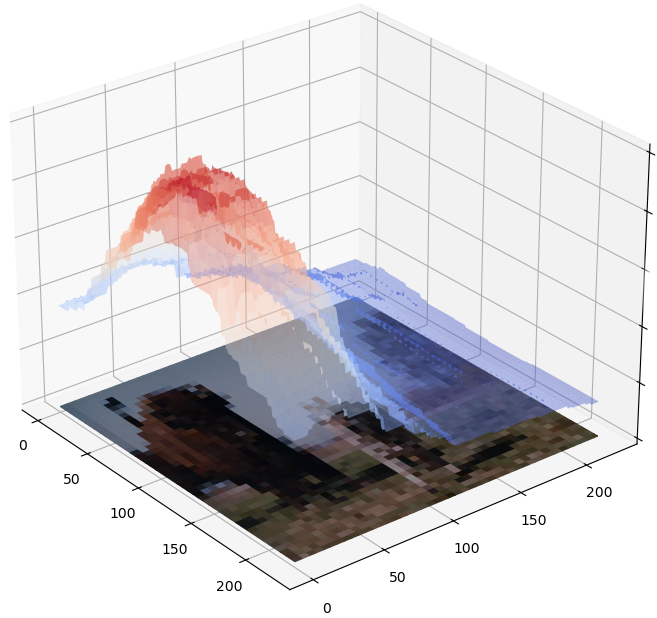}
        \caption{Bison}\label{fig:bison-map}
    \end{subfigure}
    \caption{An image with a high deletion curves does not indicate a low quality explanation. The starfish (\Cref{fig:starfish-map}) has a normalized deletion of $0.32$ because their are multiple points of high responsibility in the image. These points indicate the presence of multiple explanations. 
    The bison (\Cref{fig:bison-map}) has a low deletion of $0.1$ because there is only one point of interest in the image. 
    Deletion is a difficult measure to interpret.}
    \label{fig:deletion-curves}
\end{figure}

\paragraph*{Utility and interpretation of deletion curves}\label{subsec:delete}
\citet{petsiuk2018rise} stated that a high insertion and a low deletion is preferable. This assertion relies on the assumption that there is only one explanation for the image's classification, or that the \xai tool finds only one explanation. For deletion AUC to remain low, saliency values outside of the single `explanation' must be strictly uninformative: the saliency map cannot provide information about any other substructures present in the image. It must destroy that signal. Essentially, pixel deletion after the first discovered explanatory pixel set must be close to random, breaking up information from other potential explanations. Given that a large body of work is dedicated to `cleaning up' saliency maps, especially from \gradcam-style tools~\citep{noisetunnel,smilkov2017smoothgradremovingnoiseadding,gradcamplusplus}, it is not surprising that this genuine signal is lost.

\rex does not work like this: the responsibility map provides detailed information over the entire pixel space. Because the map is not itself the explanation, there is no motivation to remove or delete `inconvenient' hot-spots. Moreover, because \rex is not using activation, which can be noisy, every hot-spot in the \rex map is indeed associated to some degree with the desired classification. \rex utilizes the richness of its responsibility map to discover multiple, different explanations of an image~\citep{CKK25}. This is unlike the other tools in our experimental comparison, all of which return one explanation by design.

As an example, if we examine~\Cref{fig:starfish-map}, we have an image with a normalized insertion curve of $0.97$ but a normalized deletion curve of $0.32$. There would appear to be tension between the two numbers, as a high insertion suggests a high quality explanation whereas a high deletion suggests a low quality explanation. The mystery is resolved, however, by noting that the image contains multiple sufficient explanations. Indeed, there are four distinct peaks of responsibility in~\Cref{fig:starfish-map}, each of which corresponds to an independent explanation for starfish (\Cref{fig:starfish-multi}). It takes a long time to delete all of the pixels from the well-defined peaks in the responsibility map.

\begin{figure}[t]
    \centering
    \begin{subfigure}{0.19\linewidth}
        \centering
        \includegraphics[scale=0.28]{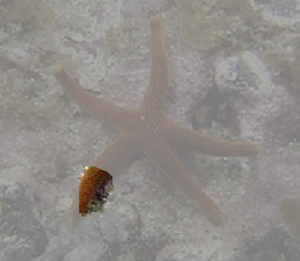}
        \caption{Explanation 1}\label{fig:starfish-exp1}
    \end{subfigure}
    \hfill
    \begin{subfigure}{0.19\linewidth}
        \centering
        \includegraphics[scale=0.28]{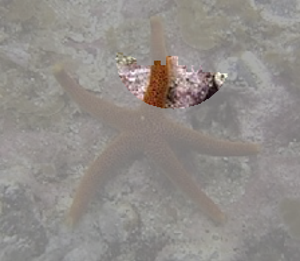}
        \caption{Explanation 2}\label{fig:starfish-exp2}
    \end{subfigure}
    \hfill
    \begin{subfigure}{0.19\linewidth}
        \centering
        \includegraphics[scale=0.28]{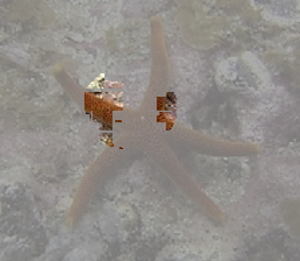}
        \caption{Explanation 3}\label{fig:starfish-exp3}
    \end{subfigure}
    \hfill
    \begin{subfigure}{0.19\linewidth}
        \centering
        \includegraphics[scale=0.28]{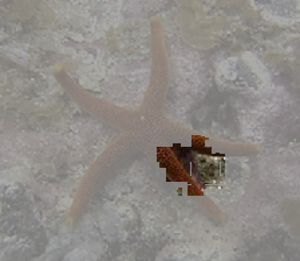}
        \caption{Explanation 4}\label{fig:starfish-exp4}
    \end{subfigure}
    \caption{The responsibility map \rex provides is sufficiently detailed to extract multiple valid explanations for most images. These starfish explanations were extracted from~\Cref{fig:starfish-map}.}
    \label{fig:starfish-multi}
\end{figure}

In~\Cref{tab:resnet50_eval,tab:vit_eval,tab:conv_eval} we follow the usual procedure of indicating the lowest deletion curve value, but stress that 
we do not consider the deletion curve as a robust measure for explanation quality.

\paragraph*{Differences between methods}
Broadly speaking, the tools can be split into three different sets. The gradient-based methods perform least well on all metrics, including overlap with a human-provided segmentation mask. The explanations they provide are diffuse and large. \noise uses \ig as its primary saliency method, and in return for a relatively small computational overhead, \noise greatly improves the explanation quality over the \ig baseline. 

\Cref{fig:cum_plots} shows cumulative plots of explanation sizes over the different tested models. Again, \rex dominates the other curves. \lime and \gradcam are, in general, the next best performing tools. The poor performance (almost linear) of \ks is interesting, suggesting that image explainability is not a natural fit for this method.

\begin{figure}[t]
    \centering
    \begin{subfigure}{0.3\linewidth}
      \centering
      \includegraphics[scale=0.25]{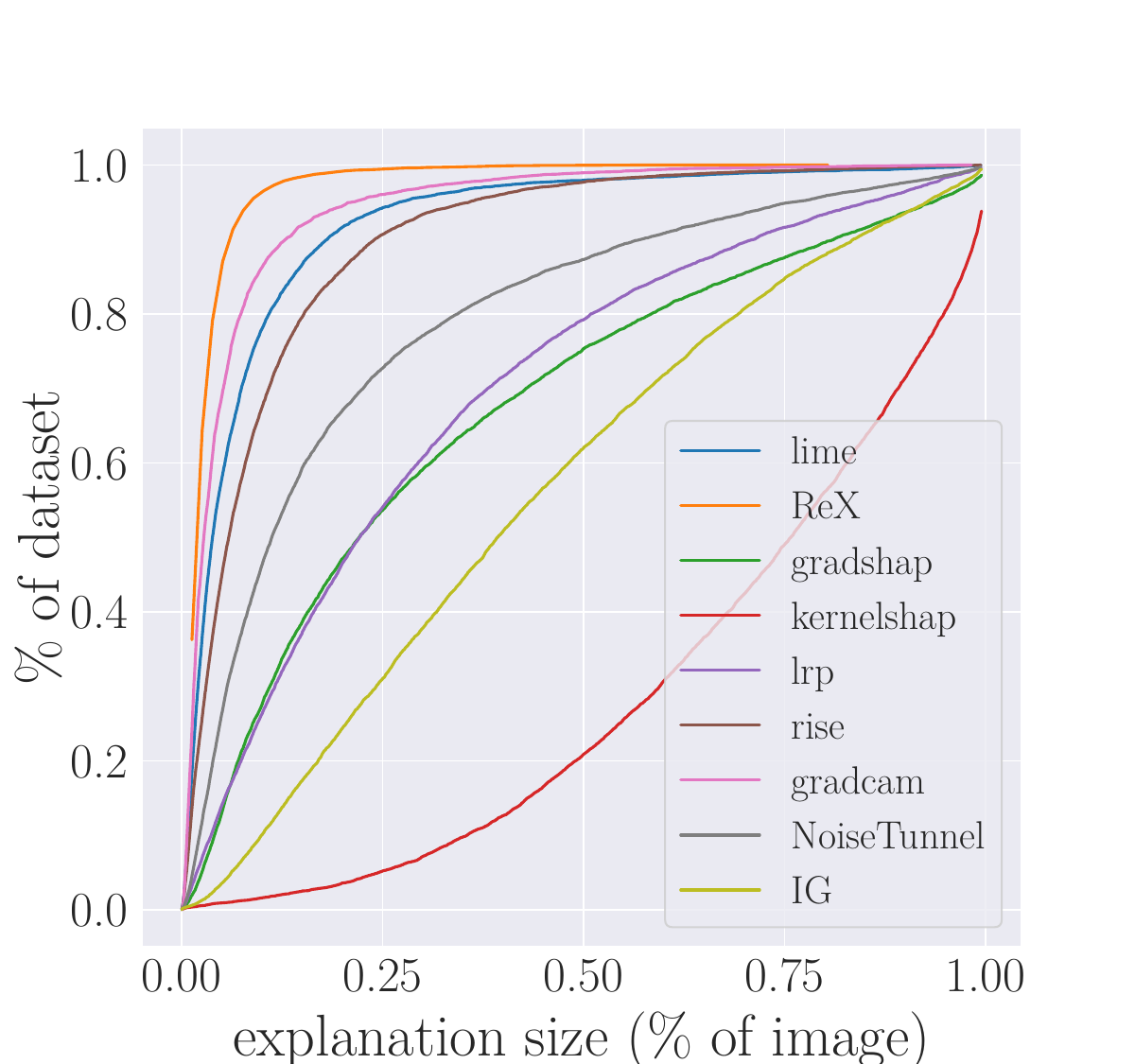}
        \caption{ResNet50}%
        \label{fig:resnet_area}  
    \end{subfigure}
    \begin{subfigure}{0.3\linewidth}
        \centering
        \includegraphics[scale=0.25]{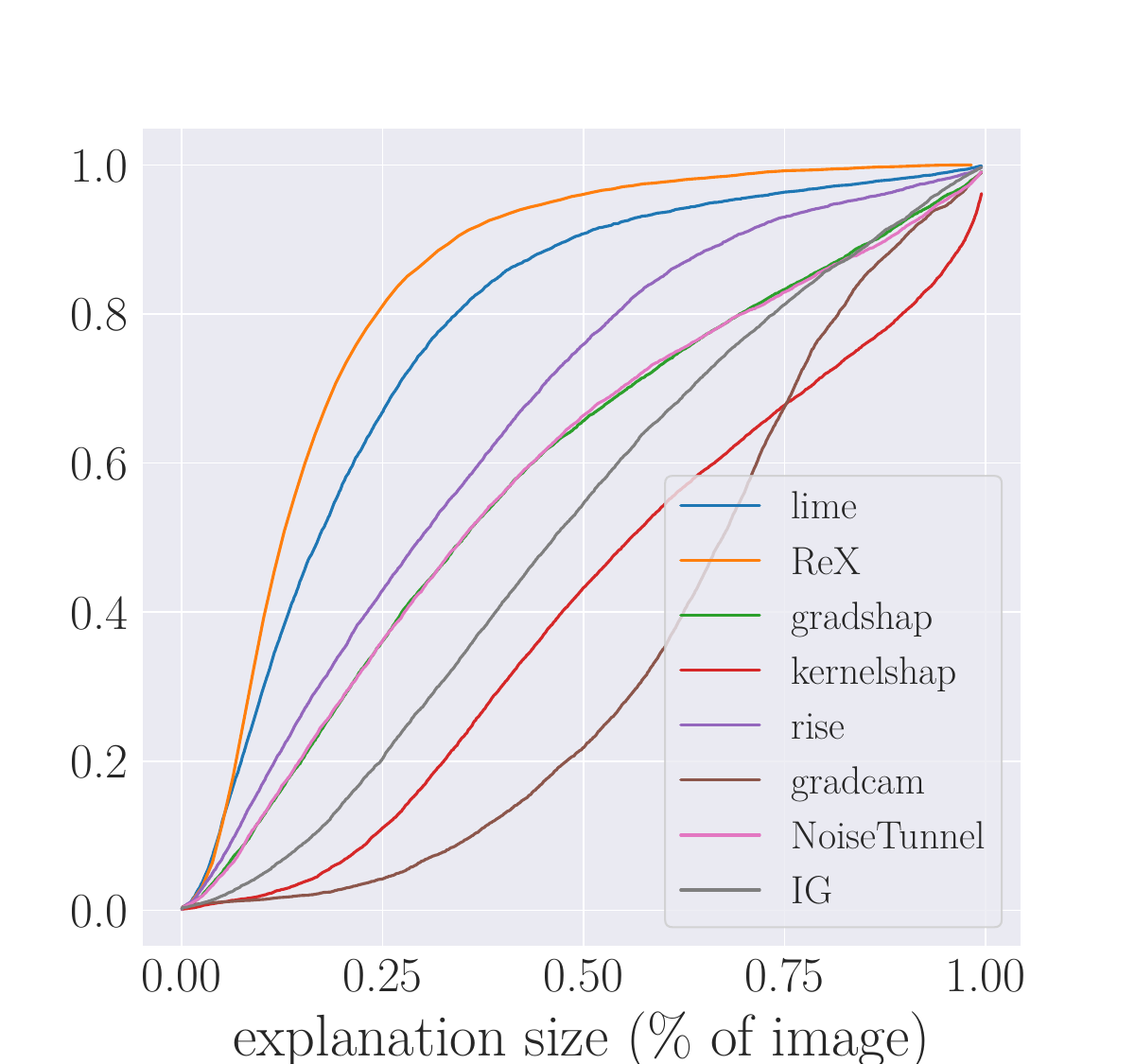}
        \caption{ViT}%
        \label{fig:vit_area}  
    \end{subfigure}
    \begin{subfigure}{0.3\linewidth}
        \centering
        \includegraphics[scale=0.25]{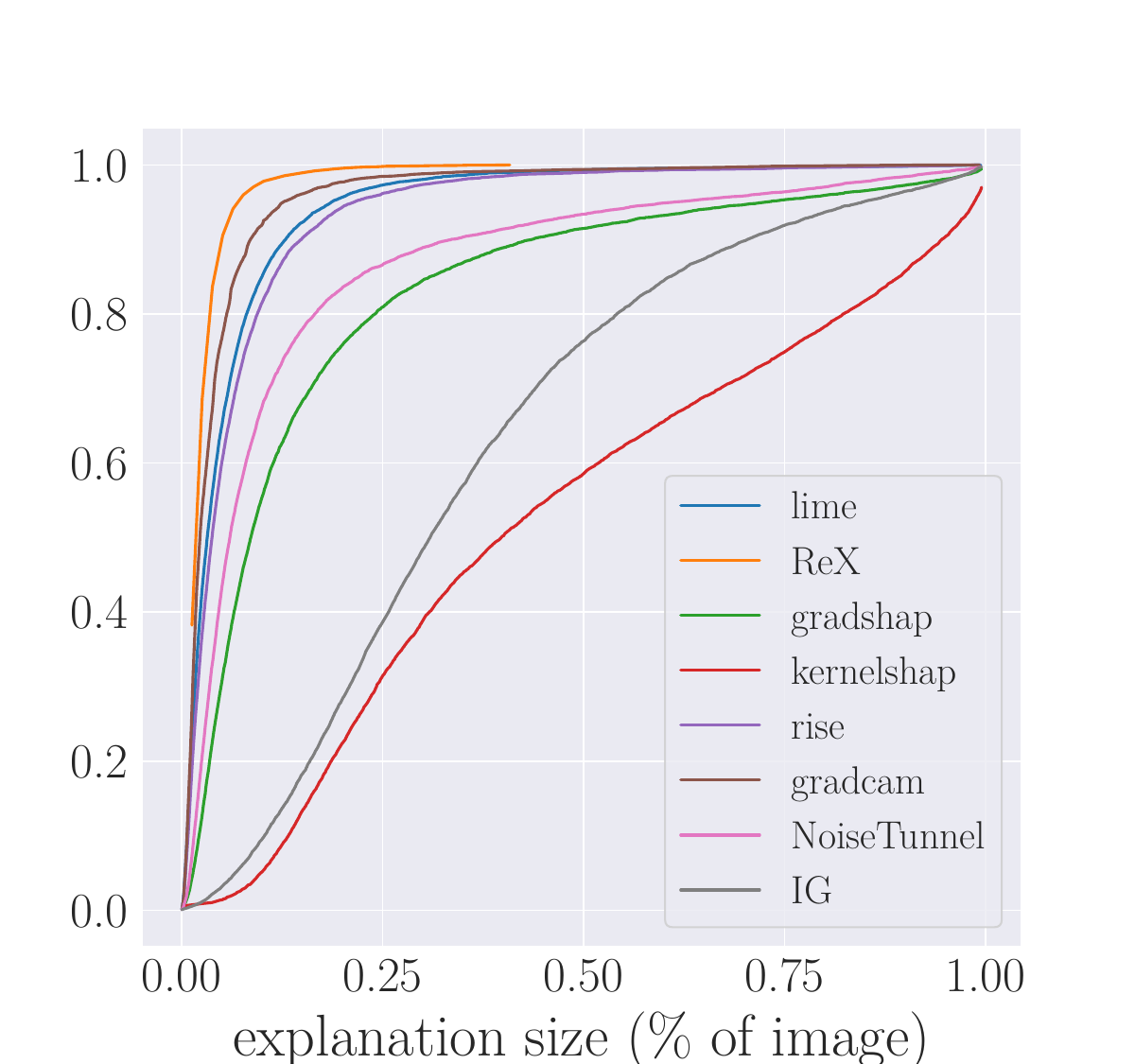}
        \caption{ConvNext}%
        \label{fig:conv_area}  
    \end{subfigure}
    \caption{Cumulative plots of explanations sizes with $3$ different models over all datasets. \rex consistently produces smaller explanations, meaning it identifies precisely those pixels required for the classification, with fewer unnecessary pixels.}
    \label{fig:cum_plots}
\end{figure}

\lime, when used for image classification, works best when provided with a segmentation mask of the image. For common datasets such as the ones we investigate, many high quality algorithms exist. \rex does not require such existing segmentation masks. \citet{blake2024explainable} conducted an investigation on brain MRI explainability using a different model and different data set and a similar, though not overlapping, set of explainability tools. On this dataset, \lime{}'s performance is hampered by medically meaningless segmentation. They found that \rex did not suffer from this problem, and performed well compared to the other tools on this different model and data. 
The reason for that is that \rex does not rely on segmentation at all. Instead, as described in \Cref{sec:algorithm}, it performs an iterative refinement by randomly partitioning
the input and averages the results over a number of different random partitions. 

The performance of \rex is notable across all datasets for being very consistent in terms of the evaluation metrics. The other tools examined are not as consistent.
Note that the values for IN and OUT in~\Cref{tab:resnet50_eval,tab:vit_eval,tab:conv_eval} are percentages of the explanation contained (or outside) the segmentation. As \rex{}'s explanations are already smaller than the others, this means that, in absolute terms, \rex also has fewer extraneous pixels. This holds because other tools (such as \lime) may have similar IN values, but a larger explanation size.

\begin{figure}
    \centering
    \subfloat[\lab{West Highland white terrier}]{
    \includegraphics[width=0.22\linewidth]{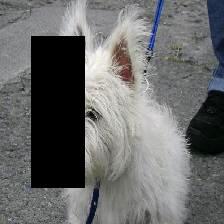}
    \includegraphics[width=0.22\linewidth]{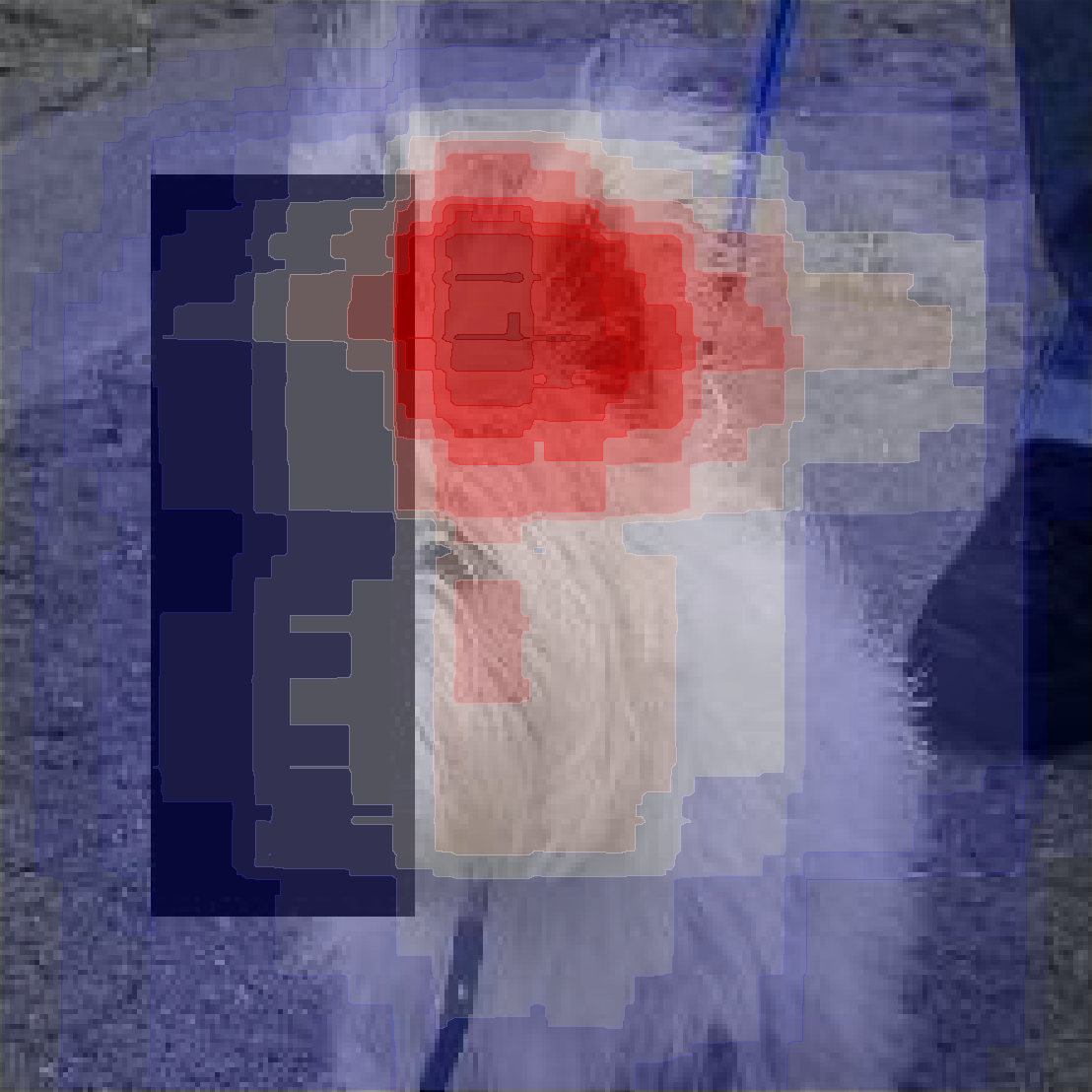}
    }
    \subfloat[\lab{Ocean liner}]{
    \includegraphics[width=0.22\linewidth]{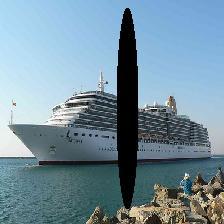}
    \includegraphics[width=0.22\linewidth]{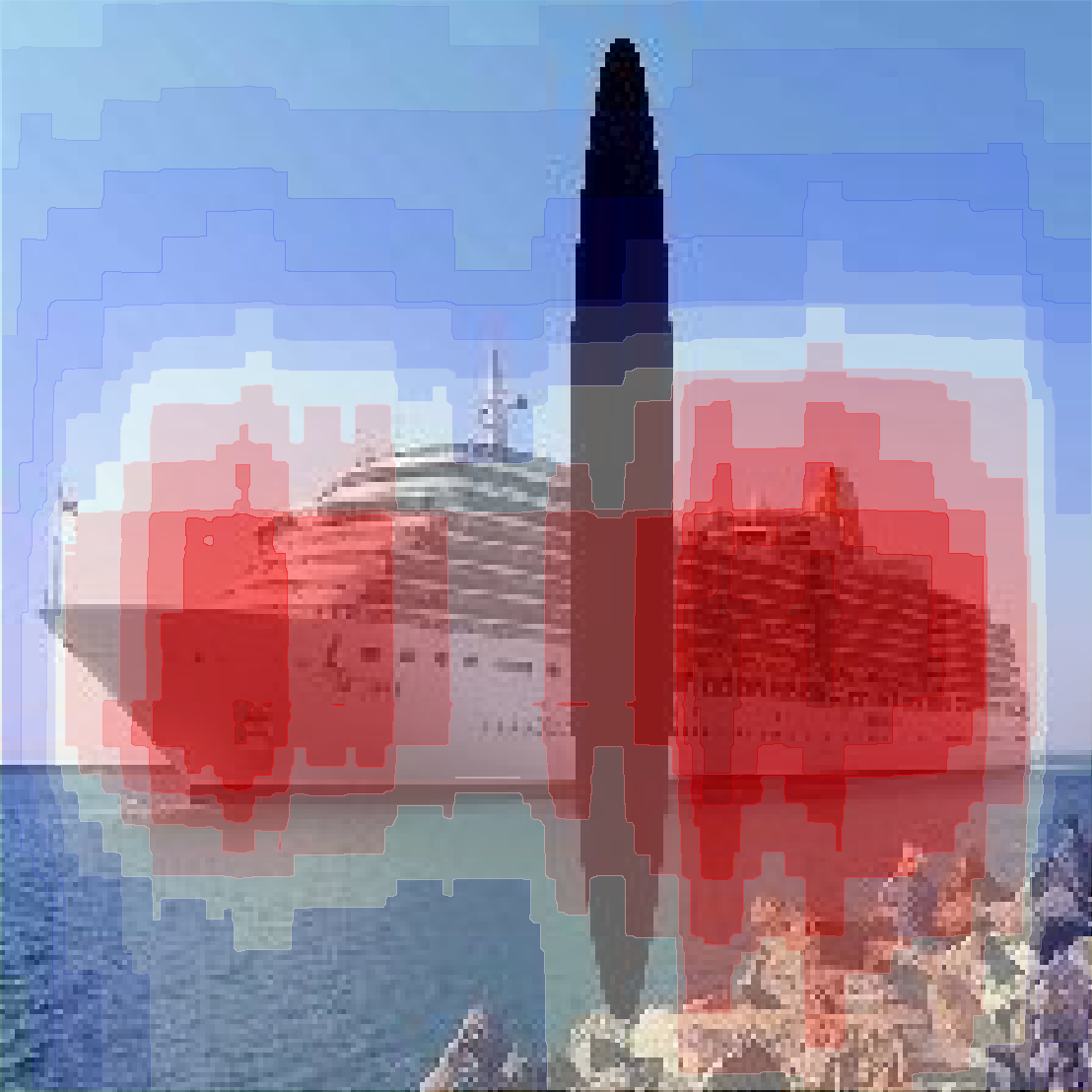}
    }
    \caption{Photo Bombing images and responsibility maps from \rex}
    \label{fig:photo-bombing-examples}
\end{figure}

We evaluated the results produced by \rex against those produced by other explainability tools on a number of standard measures proposed in the literature, such
as the size of explanations, insertion and deletion curves. For the datasets containing the ground truth or some part of it, such as the segmentation information or
the partial occlusion, we also measured the percentage of the explanation that is clearly outside of an area where it should be contained, such as the occlusion in
partially occluded areas (\Cref{fig:photo-bombing-examples}). For all other
measures, \rex shows superior performance, which answers \textbf{RQ5}.

Finally, for \textbf{RQ7}, we did not exhaustively compare the results against all possible explainability methods. Rather, we chose the state-of-the-art primary attribution methods from the Captum library.

\section{Conclusions and Future Work}
\label{sec:conclusions}

In this paper we described an approach to explainability of image classifiers that is rooted in actual causality and computes
explanations based on the formal definitions of sufficient explanations. As the exact computation is intractable, we described a modular 
algorithm for computing approximate explanations using a \emph{causal ranking} of parts of the
input image, viewing the classifier as black-box. 
Our experiments demonstrate that an implementation of our algorithm in the tool \rex produces results that are superior to 
those of state-of-the-art tools according to standard measures. 

In the future, we will apply \rex to other modalities that lend themselves to occlusion-based reasoning, such as tabular data, time series,
and spectroscopy. We will also explore different domains of application, where precise explanations are instrumental, in
particular mission-critical and safety-critical domains, such as in healthcare AI and autonomous vehicles. The healthcare domain will require
some adaptations of our algorithms, since the quality of medical images is different from the quality of general images (\eg, there are fewer
colors in an X-ray than in a general image). In the automotive domain, we will adapt \rex to analyze object detectors and also devise
new algorithms that work in (near) real-time, to address the needs of explainability for autonomous vehicles.

\section*{Acknowledgments}
A part of this research, describing explanations for partially occluded images, appeared in the Proceedings of the ICCV conference in 2021, titled ``Explanations for Occluded Images''.

Hana Chockler and David A. Kelly are partially supported by the UKRI AI program and the Engineering and Physical Sciences Research Council 
for CHAI -- Causality in Healthcare AI Hub [grant number EP/Y028856/1]. 

The authors are very grateful to Sander Beckers for his careful reading of an earlier draft and constructive suggestions regarding a simplified causal framework.

The authors are extremely grateful to Joe Halpern, who sadly passed away in February 2026. We discussed the core ideas
with Joe, and his comments shaped the way we think about causality in the context of image classification. His
intellectual curiosity and his vision continue to inspire us.\\\\

\printbibliography
\vspace{5mm}
\begin{appendices}
    \crefalias{section}{appsec}
    \section{Our Definitions in the Actual Causality Landscape}\label{app:exp}
For the ease of presentation, we first restate our definition of explanation.
\expdef*

The following is the latest definition of actual cause due to Halpern~\citep{Hal15} (see \citep{HP01b} for the original definition).
\dfn\label{def:AC}[Actual cause~\citep{Hal15}]
$\vec{X} = \vec{x}$ is 
an \emph{actual cause} of $\varphi$ in $(M,\vec{u})$ if the
following three conditions hold: 
\begin{description}
\item[{\rm AC1.}]\label{ac1} $(M,\vec{u}) \models (\vec{X} = \vec{x})$ and $(M,\vec{u}) \models \varphi$. 
\item[{\rm AC2.}] There is a
  a setting $\vec{x}'$ of the variables in $\vec{X}$, a 
(possibly empty)  set $\vec{W}$ of variables in $\V - \vec{X}'$,
and a setting $\vec{w}$ of the variables in $\vec{W}$
such that $(M,\vec{u}) \models \vec{W} = \vec{w}$ and
$(M,\vec{u}) \models [\vec{X} \gets \vec{x}', \vec{W} \gets
    \vec{w}]\neg{\varphi}$, and moreover
\item[{\rm AC3.}] \label{ac3}\index{AC3}  
  $\vec{X}$ is minimal; there is no strict subset $\vec{X}'$ of
  $\vec{X}$ such that $\vec{X}' = \vec{x}''$ can replace $\vec{X} =
  \vec{x}'$ in 
  AC2, where $\vec{x}''$ is the restriction of
$\vec{x}'$ to the variables in $\vec{X}'$.
\end{description}
\edfn

The following lemma states their equivalence in our setting.
\begin{lemma}\label{lemma-equiv}
Given a image classifier $\mathcal{N}$, an input $x$, and a causal model $M$ derived from them as
described in \Cref{sec:theory}, a subset $\vec{V}_{exp}$ of $\vec{V}$ is a single-context explanation of
$\mathcal{N}(x)$ according to \Cref{defn:simple-exp} iff $\vec{V}_{exp} = 0$ (that is, all variables in $\vec{V}_{exp}$ have the value $0$)
is an actual cause of $O=0$ in $(M,\vec{u}_0)$ according to
\Cref{def:AC}, under the assumption that a fully masked image has a different classification than $x$.
\end{lemma}
\begin{proof}
  First we observe that AC1 simply states that in the current context $\vec{u}_0$, the value of all variables in $\vec{V}$ is $0$, and the
  classification of the fully masked image is different from  $\mathcal{N}(x)$.
  Assume that EXIM1 holds. Then, assigning $\vec{V}_{exp}$ to $1$ results in restoring the original classification $\mathcal{N}(x)$, and hence $O=1$, hence
  AC2 holds with $\vec{W} = \emptyset$. For the other direction, assume that AC2 holds. Since the variables in $\vec{V}$ are causally independent, changing
  $\vec{V}_{exp}$ to $1$ does not change any other input variable, and hence we can take $\vec{W} = \emptyset$. Then, AC2 implies EXIM1.
  Finally, AC3 is simply the minimality condition, and is equivalent to EXIM2.
\end{proof}

We note that the theory of actual causality includes the notion of explanations as well~\citep{HP01a}, but these are defined over a \emph{set} of contexts, rather than
a single context, so are not suitable for our purposes. We can, in general, imagine a setting in which we are searching for a 
subset of pixels of the input
image that is sufficient to result in the original classification, regardless of the values of the rest of the pixels---this 
would be  similar to the 
setting explored in Anchors~\citep{anchors}. However, such a setting leads to significantly less intuitive results for explaining image classification, as also
evident from the results of applying Anchors to images.

Let us now place our definition of sufficient responsibility into the actual causality landscape.
For the ease of presentation, we first restate the definition.
\respdef*

In the literature, the degree of responsibility is defined by Chockler and Halpern~\citep{CH04} as the quantification of actual causality, hence its
precise definition is tied to the definition of the actual cause. The latest version is as follows.
\dfn\label{def:AR}[Degree of responsibility~\citep{Hal19}]
For $\vec{X} = \vec{x}$ being
an \emph{actual cause} of $\varphi$ in $(M,\vec{u})$, the degree of responsibility of $\vec{X} = \vec{x}$
for the value of $\varphi$ in $(M,\vec{u})$ is defined as
\[ dr(\vec{X} = \vec{x}, \varphi, M,\vec{u}) = \frac{1}{|\vec{X}|+|\vec{W}|},\]
for the smallest $|\vec{X}|+|\vec{W}|$ that satisfy \Cref{def:AC}.
\edfn

The following is immediate from the definitions, given that in our case $\vec{W} = \emptyset$.
\begin{proposition}\label{prop-resp}
Given a image classifier $\mathcal{N}$, an input $x$, a causal model $M$ derived from them as
described in \Cref{sec:theory}, and a single-context explanation $\vec{V}_{exp} \subseteq \vec{V}$ of
$\mathcal{N}(x)$ according to \Cref{defn:simple-exp}, the degree of sufficient responsibility of $\vec{V}_{exp}$
for $\varphi$ in $M$ is the degree of responsibility of $O=0$ in $(M,\vec{u}_0)$ 
according to \Cref{def:AR}, under the assumption that a fully masked image has a different classification than $x$.
\end{proposition}
Indeed, the degree of sufficient responsibility of $\vec{V}_{exp}$ is $\frac{1}{\vec{V}_{exp}}$, which is exactly
the degree of responsibility according to \Cref{def:AR} for $(M,\vec{u}_0)$, given that in our model $M$ there is
no dependency between the variables.

\section{The Complexity of Our Definitions}\label{app:comp}

In order to analyze the complexity of our definition of single-context explanation, we first need to introduce a complexity class.
\citet{PY84} introduced the complexity class $DP$, which consists of all languages $L$ such that
there exists a language $L_1$ in $NP$ and a language $L_2$ in co-$NP$ such 
that $L = L_1 \cap L_2$.
They showed that a number of problems of interest were $DP$ complete.

For the complexity discussion below, we use the classifier $\mathcal{N}$ as an oracle, not taking its complexity into account. Instead, we assume that
the value of $O$ is computed in polynomial time using the values of variables in $\vec{V}$. Since all variables in our causal models are binary,
the computation is simply a Boolean formula.

Halpern proved that the decision problem of actual causality as per \Cref{def:AC} is $DP$-complete, 
and the result holds for binary models as well~\citep{Hal15}. 
The proof, however, relies on causal models having causal dependencies between their variables. As we show below,
our definition, while still intractable, is in a lower complexity class. But first we observe that for singletons, our definition is computable in polynomial time.

\begin{proposition}
If $\vec{V}_{exp}$ is a single variable $V$, the complexity of deciding whether it satisfies \Cref{defn:simple-exp} is polynomial in the size of the input.
\end{proposition}
\begin{proof}
    The proof is straightforward by observing that EXIM1 is checkable in polynomial time, and EXIM2 holds trivially, as singletons have no subsets.
\end{proof}

We note that the decision problem of actual cause is $NP$-complete for singletons~\citep{Hal15}, indicating that our setting significantly simplifies the problem. The reduction in complexity also holds for the general case, as proved in the
following theorem.

\begin{theorem}\label{theor:exp}
The decision problem of explanations in \Cref{defn:simple-exp} is co-$NP$-complete.
\end{theorem}
\begin{proof}
For the membership in co-$NP$, we prove that the complementary problem, of showing that a given  $\vec{V}_{exp}$ is not an explanation, is in $NP$.
Given a set $\vec{V}_{exp}$, checking whether it satisfies EXIM1 is polynomial. 
For refuting EXIM2, it suffices to find one witness subset 
$\vec{V}' \subset \vec{V}$ that satisfies EXIM1, hence the complementary problem is indeed in $NP$.

For the co-NP-hardness, we show a reduction from the classic 
co-$NP$-complete problem $UNSAT = \{$ unsatisfiable propositional Boolean formulae $\}$.
Given a propositional Boolean formula $\varphi$ over the set of variables $X_1, \ldots, X_n$, we construct a depth-$2$ binary causal model $M_\varphi$ as follows. 
The set of endogenous variables $\V$ is the set $X_1, \ldots, X_n, O$, where the values of $X_1, \ldots, X_n$ are determined directly by the exogenous variables,
and 
\[ O = (\varphi \vee \ \bigwedge_{i=1}^n X_i) \wedge (\bigwedge_{i=1}^n X_i \rightarrow \neg{\varphi}). \]\label{equ:complexity}
Essentially, the value of $O$ is an exclusive OR of the value of $\varphi$ and the conjunction of all variables of $\varphi$.
Then, $\varphi \in UNSAT$ iff the set $\vec{X} = \{X_1, \ldots X_n\}$ is a single-context explanation for $O=0$ in $(M,\vec{u}_0)$.

First, we claim that $O$ has the value $1$ only in the assignment that sets all variables in $\vec{X}$ to $1$ 
iff $\varphi$ is unsatisfiable.
Indeed, if $\varphi$ is unsatisfiable, then its value is $0$ under all assignments, and the value of the conjunction is $1$ iff
all variables in $\vec{X}$ are set to $1$. 
For the other direction, if the value of $O$ is $1$ under the assignment that sets all variables in $\vec{X}$ to $1$,
then $\varphi$ is falsified under this assignment, and since this is the only assignment that results in $O=1$, 
$\varphi$ is falsified under
all other assignments to $\vec{X}$ as well, hence $\varphi$ is unsatisfiable.

We are now ready to prove the reduction.
Assume first that $\varphi \in UNSAT$. Then, by the previous discussion, $(M,\vec{u}_0) \models [\vec{X} \leftarrow 1](O=1)$,
satisfying EXIM1, and $\vec{X}$ is the only such subset of variables, satisfying EXIM2. Therefore, $\vec{X}$ is a single-context explanation
to $O$ in $M_\varphi$, as required.

For the other direction, assume that $\vec{X}$ is a single-context explanation to $O$ in $M_\varphi$. Then, 
$(M,\vec{u}_0) \models [\vec{X} \leftarrow 1](O=1)$, and no subset of $\vec{X}$ satisfies this condition, hence by the previous discussion,
$\varphi \in UNSAT$, completing the proof.
\end{proof}
There is, therefore, little hope to find an efficient algorithm for computing exact explanations for image classification, and the answer to \textbf{RQ2} wrt 
precise explanations is `no', as computing explanations is intractable. 

While this is not directly relevant to this paper, the following observation is a direct corollary from \Cref{lemma-equiv} and \Cref{theor:exp}.
\begin{observation}
In depth-2 binary causal models, the decision problem of actual causality is co-$NP$-complete.
\end{observation}

We now turn to examining the complexity of the degree of sufficient responsibility. First, we need to introduce a new complexity hierarchy, specifically,
the hierarchy of functional problems. This is because the output of the degree of the responsibility computation is a value, and not an accept/reject decision.
For a complexity class $A$, the class \fpa consists of all functions that can be computed 
by a polynomial-time Turing machine with an oracle for a problem in
$A$, which on input $x$ asks a total of $O(\log{|x|})$ queries 
(cf.~\citep{Pap84}).
A function $f(x)$ is \fpa\-hard iff for every function $g(x)$
in \fpa \ there exist polynomially computable functions
$R,S: \Sigma^* \rightarrow \Sigma^*$ (where $\Sigma$ is the common alphabet)
such that $g(x) = S(f(R(x)))$. A function $f(x)$ is complete in \fpa iff
it is in \fpa and is \fpa-hard.

We prove the following result.
\begin{theorem}\label{theor:resp}
The problem of computing the degree of sufficient responsibility is \fp-complete.
\end{theorem}
\begin{proof}
The proof is very similar to the proof of complexity of the degree of responsibility by~\citet{CH04}, so we present only the key points here.

For the membership in \fp, we observe that the polynomial time algorithm for computing the degree of responsibility can perform a binary search on the size of
$\vec{V}_{exp}$, in each step querying an oracle for the existence of a single-context explanation $\vec{V}_{exp}$ of size $k$. By ~\Cref{theor:exp}, the
decision problem of the existence of explanation is co-$NP$-complete, and hence can be decided by an $NP$-oracle.

For hardness in \fp, the reduction is from the \fp-complete function problem $MINSAT$, defined in~\citep{CH04} as $MINSAT(\varphi) = k$, where $k$ is the minimal number of $1$'s
in a satisfying assignment for $\varphi$. The reduction is the same as in~\citep{CH04}.
\end{proof}

\Cref{theor:resp} shows that computing the degree of sufficient responsibility in our setting is intractable, hence the need for approximation algorithms.

\end{appendices}

\end{document}